\documentclass{article}

\usepackage{microtype}
\usepackage{graphicx}
\usepackage{subcaption}
\usepackage{booktabs} 

\usepackage{hyperref}

\usepackage{multirow}
\usepackage{epsfig}
\usepackage{soul}

 \usepackage[accepted]{icml2024}

\usepackage{amsmath}
\usepackage{amssymb}
\usepackage{mathtools}
\usepackage{amsthm}
\usepackage{pifont}
\usepackage[capitalize,noabbrev]{cleveref}

\theoremstyle{plain}
\newtheorem{theorem}{Theorem}[section]
\newtheorem{proposition}[theorem]{Proposition}
\newtheorem{lemma}[theorem]{Lemma}

\theoremstyle{definition}

\newtheorem{assumption}[theorem]{Assumption}
\theoremstyle{remark}

\usepackage[textsize=tiny]{todonotes}

\usepackage{amsmath,amsfonts,bm}

\def\eqref#1{equation~\ref{#1}}

\def\1{\bm{1}}

\def\vtheta{{\bm{\theta}}}

\def\vm{{\bm{m}}}

\def\vx{{\bm{x}}}

\def\mJ{{\bm{J}}}

\def\mM{{\bm{M}}}

\DeclareMathAlphabet{\mathsfit}{\encodingdefault}{\sfdefault}{m}{sl}
\SetMathAlphabet{\mathsfit}{bold}{\encodingdefault}{\sfdefault}{bx}{n}

\def\gL{{\mathcal{L}}}
\def\gM{{\mathcal{M}}}

\def\gS{{\mathcal{S}}}

\icmltitlerunning{Towards Interpretable Deep Local Learning with Successive Gradient Reconciliation}

\begin{document}

\twocolumn[
\icmltitle{Towards Interpretable Deep Local Learning with Successive Gradient Reconciliation}  

\vspace{0mm}
\begin{icmlauthorlist}
	\icmlauthor{Yibo Yang}{yyy}
	\icmlauthor{Xiaojie Li}{xxx,pcl}
	\icmlauthor{Motasem Alfarra}{yyy}
	\icmlauthor{Hasan Hammoud}{yyy}
	\icmlauthor{Adel Bibi}{oxf}
	\icmlauthor{Philip Torr}{oxf}
	\icmlauthor{Bernard Ghanem}{yyy}
\end{icmlauthorlist}
\vspace{-1mm}

\icmlaffiliation{yyy}{King Abdullah University of Science and Technology (KAUST)}
\icmlaffiliation{xxx}{Harbin Institute of Technology (Shenzhen)}
\icmlaffiliation{oxf}{University of Oxford}
\icmlaffiliation{pcl}{Peng Cheng Laboratory}

\icmlcorrespondingauthor{Yibo Yang}{yibo.yang93@gmail.com}

\icmlkeywords{Machine Learning, ICML}

\vskip 0.3in
]



\printAffiliationsAndNotice{}  

\begin{abstract}
	\vspace{-1mm}
Relieving the reliance of neural network training on a global back-propagation (BP) has emerged as a notable research topic due to the biological implausibility and huge memory consumption caused by BP. Among the existing solutions, local learning optimizes gradient-isolated modules of a neural network with local errors and has been proved to be effective even on large-scale datasets. However, the reconciliation among local errors has never been investigated. In this paper, we first theoretically study non-greedy layer-wise training and show that the convergence cannot be assured when the local gradient in a module \emph{w.r.t.} its input is not reconciled with the local gradient in the previous module \emph{w.r.t.} its output. Inspired by the theoretical result, we further propose a local training strategy that successively regularizes the gradient reconciliation between neighboring modules without breaking gradient isolation or introducing any learnable parameters. Our method can be integrated into both local-BP and BP-free settings. In experiments, we achieve significant performance improvements compared to previous methods. 
Particularly, our method for CNN and Transformer architectures on ImageNet is able to attain a competitive performance with global BP, saving more than 40\% memory consumption.
\end{abstract}

\vspace{-7mm}
\section{Introduction}
\label{introduction}


Back-propagation (BP) has been a crucial ingredient for the success of deep learning \cite{lecun2015deep}. Although BP is easy to implement and seems indispensable for training deep neural networks, it has been a concern that BP is distinct from how the brain learns and updates, known as biological implausibility \cite{lillicrap2020backpropagation,bengio2015towards}. First, BP encounters the weight transport problem \cite{grossberg1987competitive}, which means the update of each layer during backward propagation relies on the symmetric weight used for forward propagation \cite{lillicrap2016random,liao2016important}. Second, compared to the brain that updates neurons instantly using local signals, training neural networks with BP suffers from update locking because the gradient descent in a layer can be only performed after the forward and backward propagations of its subsequent layers \cite{jaderberg2017decoupled,frenkel2021learning,dellaferrera2022error,halvagal2023combination}. 
Moreover, due to update locking, all the activations and gradients of a whole model need to be stored, which is the dominant cause of the huge memory consumption for training modern neural networks whose capacity is continually expanding.  

In order to tackle these issues, different approaches are proposed to relieve the reliance on a global BP \cite{nokland2016direct,clark2021credit,silver2021learning,ren2023scaling,journ2023hebbian,belilovsky2019greedy,wang2021revisiting,fournier2023preventing}, among which training with local errors is promising due to its less impaired performance. It divides a neural network into several gradient-isolated modules and trains each locally. The initial explorations adopt greedy layer-wise training for a good initialization before finetuning with BP \cite{hinton2006fast,bengio2006greedy}. Later studies show that greedy layer-wise training can achieve competitive performance on the large scale ImageNet using auxiliary classifiers \cite{belilovsky2019greedy}. Since learning sequentially does not enable to optimize the deep representation simultaneously, recent studies favor the non-greedy fashion where each layer is updated locally with a mini-batch data and then passes the output into the next layer \cite{belilovsky2020decoupled,siddiqui2023blockwise,wang2021revisiting}. However, a performance drop is still inevitable when increasing the number of local modules, and current studies have not rigorously pinpointed the inherent limitation of local training compared to global BP.



In this paper, we investigate the defect of \emph{non-greedy} layer-wise training from a view of the reconciliation among local errors and propose a remedy for it. Consider a local layer $x_{k}= f(x_{k-1}, \theta_k)$, where $x_{k-1}$ is the input of layer $k$ and also the output of layer $k-1$, and $\theta_k$ is the learnable parameters of this layer. Local training optimizes $\theta_k$ independently with a classifier head $l_k=h(x_k,w_k)$ that is parameterized by $w_k$ to produce the local error signal $l_k$ using a loss function, \emph{e.g.} cross entropy loss. Note that the local layer $f$ is a function with two variables $x_{k-1}$ and $\theta_k$. 
If $x_k$ is the last-layer representation and we use global BP to train the network, 
the gradient \emph{w.r.t.} the first variable, $\frac{\partial l_k}{\partial x_{k-1}}$, can be back-propagated to prior layers whose update can lead to a new $x_{k-1}$ that reduces $l_k$. In layer-wise training, the optimization regarding $x_{k-1}$ is achieved by updating $\theta_{k-1}$ with the previous local error $l_{k-1}$. However, there is no assurance that the update of $\theta_{k-1}$ in the previous layer can cause a change of $\Delta x_{k-1}$ that most satisfies the demand of the current layer to reduce $l_k$. A previous study \cite{jaderberg2017decoupled} deals with the correctness of local gradients by learning local gradient generators. But it relies on the global true gradient so breaks gradient isolation. We point out that the discordant local update, which global BP is naturally immune to, is the fundamental defect in the context of non-greedy layer-wise training. 

To this end, we first theoretically analyze the convergence of a two-layer model. The two layers are updated with their respective local errors in a non-greedy manner. Our result indicates that when the gradient of the second layer \emph{w.r.t.} its input, \emph{i.e.,} $\frac{\partial l_2}{\partial x_1}$, has a large distance to the gradient of the first layer \emph{w.r.t.} its output, \emph{i.e.,} $\frac{\partial l_1}{\partial x_1}$, the convergence of the second layer cannot be guaranteed. Inspired by the result, we propose a simple yet interpretable and effective method, named successive gradient reconciliation.
When training the $(k-1)$-th layer locally, we store the gradient \emph{w.r.t.} the output $\frac{\partial l_{k-1}}{\partial x_{k-1}}$, and make the output require gradient being the input of the next layer. At the update of the $k$-th layer, we calculate the gradient \emph{w.r.t.} the input $\frac{\partial l_{k}}{\partial x_{k-1}}$, and add a regularizer on the local error to minimize the distance between the two gradients. By doing so, the optimization of each layer is towards a direction such that the parameters $\theta_k$ 
not only decrease the local error $l_{k}$, but also enable the change of input $\Delta x_{k-1}$ coming from the previous layer to decrease $l_{k}$. The process is performed layer by layer to successively reconcile local objectives and transmit them into the last layer for a better final representation without breaking gradient isolation. 

Our method effectively improves the performance of local training in both local-BP and BP-free cases. In local-BP training, we successfully enable memory-efficient training by removing a global BP. In BP-free training, we adopt a fixed local classifier head borrowing the ideas from \cite{frenkel2021learning,yang2022inducing}, and achieve significantly better performance than existing BP-free methods.
The contributions of this study can be listed as follows:
\begin{itemize}
	\item We derive a theoretical convergence bound in the context of non-greedy layer-wise training, and show that the reconciliation among local errors is crucial to ensure the convergence of the last-layer error.
	\item We propose successive gradient reconciliation, which successively reconciles the local updates of every two neighboring layers without breaking gradient isolation or introducing any learnable parameters.  
	We integrate our method in both local-BP and BP-free settings. 
	\item We conduct extensive experiments on CIFAR-10, CIFAR-100, and ImageNet to verify the effectiveness of our method, and show that our method surpasses previous methods and is able to attain a competitive performance with global BP saving over 40\% memory consumption for CNN and Transformer architectures. 
\end{itemize}

\section{Related Work}


Back-propagation has been important for deep network training due to its effective credit assignment into hierarchical representations \cite{rumelhart1986learning}. 
Despite the success, the weight transport and the update locking problems constrain BP from being biologically plausible and memory efficient \cite{lillicrap2020backpropagation,crick1989recent,grossberg1987competitive}, which has attracted wide research interests from both neuroscience and machine learning communities to propose alternatives to a global BP \cite{bengio2015towards,hinton2022forward,halvagal2023combination,song2024inferring}. 

\textbf{Gradient estimation.} Introducing auxiliary optimization variables for each layer can spare the need to back propagation a global error \cite{taylor2016training,li2020training}. Gradient estimation can be also achieved by forward automatic differentiation \cite{baydin2022gradients,silver2021learning,ren2023scaling} and forward propagation again with a disturbed input \cite{hinton2022forward,dellaferrera2022error}. However, these methods have not been scalable on large datasets and are not competitive with BP. 

\textbf{Credit assignment.} In order to relax the weight transport problem, different credit assignment strategies are proposed with only sign symmetry \cite{liao2016important,xiao2018biologicallyplausible} or random feedback connection, known as feedback alignment (FA) \cite{lillicrap2016random}. Later studies improve over FA by adjusting the random connection \cite{akrout2019deep} and directly propagating the global error into each layer \cite{nokland2016direct,clark2021credit}. 
However, most of these methods cannot achieve satisfactory performance on large datasets \cite{bartunov2018assessing} and local updates still cannot be performed without waiting. DRTP \cite{frenkel2021learning} refrains from update locking by assigning local errors directly from labels instead of the global error, but further introduces performance deterioration. 

\textbf{Local learning.} Another promising solution, which our method can be categorized as, is to look for a local update rule such that layers are optimized independently to satisfy both weight-transport free and update unlocking. Hebbian plastic rule based on synaptic plasticity has been leveraged for local update in an unsupervised or self-supervised manner \cite{illing2021local,journ2023hebbian,halvagal2023combination}. Greedy layer-wise training is initially proved to be effective in producing a good initialization \cite{yang2022towards} for the subsequent finetuning with BP \cite{hinton2006fast,bengio2006greedy}. Later studies adopt auxiliary heads to train layers locally for supervised learning \cite{mostafa2018deep,nokland2019training,belilovsky2019greedy,belilovsky2020decoupled,you2020l2} and self-supervised learning \cite{lowe2019putting,xiong2020loco,siddiqui2023blockwise}. Although there is still back-propagation within the auxiliary head composed of multiple layers, it helps to attain a comparable performance with global BP on ImageNet \cite{belilovsky2019greedy}. The recent well-performing methods favor non-greedy layer-wise training \cite{wang2021revisiting,siddiqui2023blockwise}. However, the fundamental defect of layer-wise training over global BP has not been unveiled to be clear. 
In \cite{jaderberg2017decoupled}, a gradient generator is proposed to correct local errors with the global BP signal, but it breaks the gradient isolation among local modules. 
In \cite{wang2021revisiting}, a reconstruction loss is added to each local error, but the decoder increases the length of local BP and incurs more local parameters and memory cost. 
Different from these studies, our method reconciles the local errors successively without breaking gradient isolation or introducing learnable parameters, and can be applied in BP-free training. 
A previous work studies the convergence of layer-wise training \cite{shin2022effects}, but the analysis is conducted in a linear model with global BP instead of local learning. Memory reduction can be also achieved by gradient checkpoint or reversible architecture \cite{chen2016training,gomez2017reversible}, but these strategies rely on the tradeoff between computation and memory. As a comparison, local learning can detach each layer so naturally brings memory efficiency. 




\section{Method}

In Sec. \ref{theoretical_result}, we analyze local learning and derive a convergence bound that shows its defect. Based on the result, we propose a method to remedy in Sec. \ref{method}, and show how to apply our method in local-BP training and BP-free training in Sec. \ref{training}. Finally, in Sec. \ref{justification}, we use a simplified two-layer model to empirically showcase how our method facilitates the loss reduction of the output layer. 

\subsection{Theoretical Result}
\label{theoretical_result}

We consider a two-layer model composed of $\vx_1 = f(\vx_0,\vtheta_1)$ and $\vx_2=f(\vx_1, \vtheta_2)$, where $\vtheta_1$ and $\vtheta_2$ are the parameters of the two layers, respectively, $\vx_0$ is the input data, $\vx_1$ is the output of the first layer and also the input of the second layer, $\vx_2$ is the output of the second layer, and $f$ can be a composite function including linear and non-linear transformations. In layer-wise training, we have two local heads that produce local errors $\gL_1(\vx_1)$ and $\gL_2(\vx_2)$ to update the first layer and the second layer, respectively. Different from a global BP, the gradient of $\gL_2$ does not back propagate into the first layer. Accordingly, the training in the $i$-th iteration can be formulated as:
\begin{align}
	&\vtheta_1^{(i+1)} \leftarrow \vtheta_1^{(i)} -\eta_1^{(i)} \nabla_{\vtheta_1}\gL_1(\vtheta_1^{(i)}), \label{update1}\\
	&\vtheta_2^{(i+1)} \leftarrow \vtheta_2^{(i)} -\eta_2^{(i)} \nabla_{\vtheta_2}\gL_2(\vtheta_1^{(i+1)}, \vtheta_2^{(i)}), \label{update2} 
\end{align}
where $\eta_1$, $\eta_2$ are the learning rates of the two layers. 
We analyze the convergence of the above local learning based on the PL condition \cite{karimi2016linear} and gradient Lipschitz. The assumptions and results are stated as following.

\begin{assumption}[PL Condition]
	\label{assump1}
	Let $\gL_2^*$ be the optimal function value of the second layer loss $\gL_2$. There exists a $\mu$ such that $\forall \vtheta_1, \vtheta_2$, we have
	\begin{align}
		\left\| \nabla \gL_2(\vtheta_1, \vtheta_2) \right\|^2&=\left\| \nabla_{\vtheta_1} \gL_2(\vtheta_1, \vtheta_2) \right\|^2+\left\| \nabla_{\vtheta_2} \gL_2(\vtheta_1, \vtheta_2) \right\|^2 \notag \\
		&\ge \mu\left(  \gL_2(\vtheta_1, \vtheta_2) - \gL_2^* \right), \notag
	\end{align}
	where $\nabla\gL_2=[\nabla_{\vtheta_1}\gL_2; \nabla_{\vtheta_2}\gL_2]$.
\end{assumption}

\begin{assumption}[Layer-wise Lipschitz and global Lipschitz]
	\label{assump2}
	There exists $L_1$, $L_2$, and $L > 0$ such that for all $\vtheta_{1,a}, \vtheta_{1,b}, \vtheta_1, \vtheta_{2,a}, \vtheta_{2,b}, \vtheta_2$, we have
	\begin{align}
	&\left\| \nabla_{\vtheta_1} \gL_2(\vtheta_{1,a}, \vtheta_2) - \nabla_{\vtheta_1} \gL_2(\vtheta_{1,b}, \vtheta_2) \right\| \le L_1 \left\| \vtheta_{1,a} - \vtheta_{1,b} \right\|,  \notag\\
	&\left\| \nabla_{\vtheta_2} \gL_2(\vtheta_{1}, \vtheta_{2,a}) - \nabla_{\vtheta_2} \gL_2(\vtheta_{1}, \vtheta_{2,b}) \right\| \le L_2 \left\| \vtheta_{2,a} - \vtheta_{2,b} \right\| \notag,
	\end{align}	
	and
	\vspace{-1mm}
	\begin{align}
	\left\| \nabla \gL_2(\vtheta_{1,a}, \vtheta_{2,a}) - \nabla \gL_2(\vtheta_{1,b}, \vtheta_{2,b}) \right\| \le L \left\|
	\begin{bmatrix}
		\vtheta_{1,a} - \vtheta_{1,b} \\
		\vtheta_{2,a} - \vtheta_{2,b}
	\end{bmatrix} 
	\right\|. \notag
	\end{align}
\end{assumption}

\begin{theorem}
	\label{theorem}
	Based on Assumptions \ref{assump1} and \ref{assump2}, if the learning rates are set as $\eta_1^{(i)}=\eta_1$ and $\eta_2^{(i)}=\eta_2$, where
	\begin{equation}
	\left\{\hspace{-4mm}
	\begin{array}{cl}
	& 0<\eta_1\le \min \left( \frac{\sqrt{L_1^2+8L^2}-L_1}{4L^2}, \frac{2}{\mu}, \frac{1}{2L_2} \right)\\
	& \max \left(0, \frac{1-\sqrt{1-2L_2\eta_1}}{L_2} \right) < \eta_2 \le \frac{1+\sqrt{1-2L_2\eta_1}}{L_2} \notag
	\end{array}
	\right.
	,
	\end{equation}
	we have the following convergence
	\begin{equation}
		\label{convergence1}
		\gL_2^{(i+1,i+1)}-\gL_2^* \le \left(1-\alpha\mu\right)\left(\gL_2^{(i,i)}-\gL_2^*\right)
		+ \alpha \left\| \bm{\epsilon}^{(i)} \right\|^2,
	\end{equation}
	and recursively applying Eq. (\ref{convergence1}) we have
	\begin{align}
		\label{convergence2}
		\gL_2^{(i+1,i+1)}-\gL_2^* \le& \left(1-\alpha\mu\right)^{i+1}\left(\gL_2^{(0,0)}-\gL_2^*\right) \notag \\
		&+ \alpha \sum_{k=0}^{i}(1-\alpha\mu)^k\left\| \bm{\epsilon}^{(i-k)} \right\|^2,
 	\end{align}
 	where $\gL_2^{(i,i)}$ denotes $\gL_2(\vtheta_1^{(i)}, \vtheta_2^{(i)})$, \emph{i.e.,} the second layer loss value with parameters $\vtheta_1^{(i)}$ and $\vtheta_2^{(i)}$ in the $i$-th iteration, $\alpha=\frac{\eta_1}{2}$, and $\bm{\epsilon}^{(i)} \triangleq \nabla_{\vtheta_1}\gL_2^{(i,i)} - \nabla_{\vtheta_1}\gL_1^{(i)}$.  
\end{theorem}

\begin{figure*}[t]
	\centering
	\includegraphics[width=0.88\linewidth]{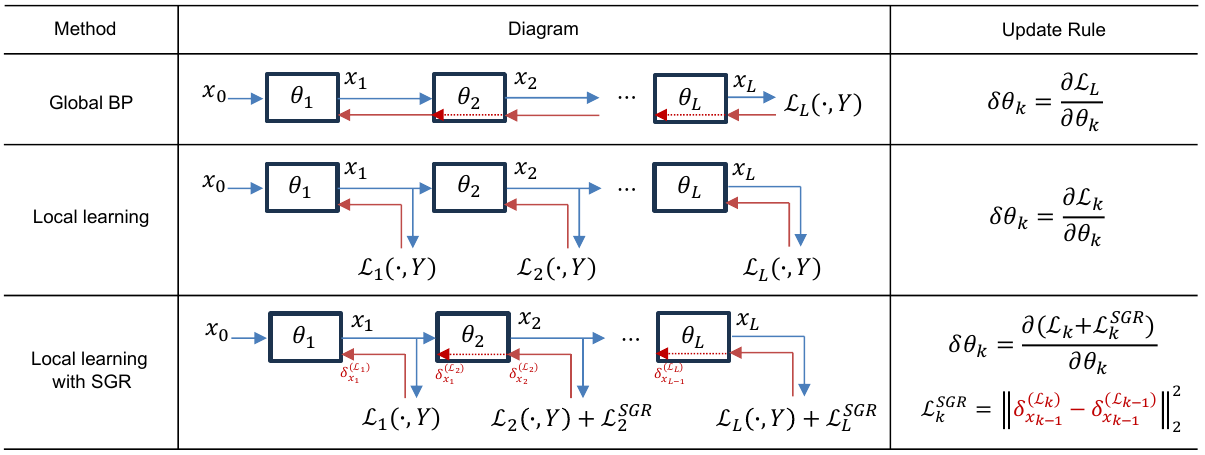}
	\vspace{-2mm}
	\caption{An illustration to compare our method with non-greedy local learning and global BP. The blue arrows indicate forward propagation, while the red solid arrows and red dashed arrows denote the backward gradient \emph{w.r.t.} the output feature and the input feature of each block, respectively. In global BP, gradients are passed into prior blocks to update the parameters, but the updates in local learning  may be deviated by local errors. Our method successively reconciles local updates in a forward mode without breaking gradient isolation.}
	\label{fig1}
	\vspace{-1mm}
\end{figure*}

We care more about $\gL_2$ because it is the loss of the last layer that is used to produce the representation for inference. 
Note that in local learning, the update of $\vtheta_{1}$ in Eq. (\ref{update1}) is only dependent on $\gL_1$ due to gradient isolation. 
The result in Eq. (\ref{convergence1}) indicates that when the local gradient $\nabla_{\vtheta_1}\gL_1^{(i)}$ has a large distance with $\nabla_{\vtheta_1}\gL_2^{(i,i)}$, \emph{i.e.,} $\left\| \bm{\epsilon}^{(i)} \right\|$ is large, the optimization of the $i$-th iteration does not ensure a loss reduction of $\gL_2$. Moreover, $\bm{\epsilon}$ has an accumulative effect during training, as shown in Eq. (\ref{convergence2}). As a result, only when $\left\| \bm{\epsilon} \right\|$ in every iteration keeps small, can we ensure that $\gL_2$ is converged towards its optimality. 

\subsection{Successive Gradient Reconciliation}
\label{method}

Even though learning by local update rules has been widely adopted in prior studies \cite{ren2023scaling,siddiqui2023blockwise,frenkel2021learning,belilovsky2020decoupled}, our theoretical result in Theorem \ref{theorem} implies that a defect of local learning, which may impede its convergence, lies in the reconciliation between local gradient and the global one. 
As a comparison, when training with global BP, the local gradient in Eq. (\ref{update1}), $\nabla_{\vtheta_1}\gL_1^{(i)}$, is replaced by the global gradient, $\nabla_{\vtheta_1}\gL_2^{(i,i)}$, via back propagation, which means $\left\| \bm{\epsilon} \right\|=0$. Therefore, global BP is inherently spared from discordant updates. In \cite{jaderberg2017decoupled}, a generator is learned to produce a synthetic local gradient close to the global one. But it relies on the global BP to provide the true gradient, so breaks gradient isolation.

We propose successive gradient reconciliation (SGR) for local learning. It is able to reconcile local updates without relying on the true global gradient or breaking gradient isolation. For any two neighboring layers, we have
\begin{align}
	\left\| \nabla_{\vtheta_{k}}\gL_{k+1} - \nabla_{\vtheta_{k}}\gL_{k} \right\| 
	& \le \left\| \mJ_f(\vtheta_k)^T \right\| \left\| \delta_{\vx_{k}}^{(\gL_{k+1})} - \delta_{\vx_{k}}^{(\gL_{k})} \right\|,  \notag 
\end{align}
where $\mJ_f(\vtheta_k)$ is the Jacobian matrix of the output feature $\vx_k$ \emph{w.r.t.} the parameters $\vtheta_k$ in this layer, and $\delta_{\vx_{k}}^{(\gL_{k+1})}$ is the abbreviation for $\frac{\partial \gL_{k+1}}{\partial \vx_{k}}$, \emph{i.e.,} the gradient \emph{w.r.t.} $\vx_{k}$ from the next layer's loss $\gL_{k+1}$. Instead of directly minimizing the left side of the above equation, which needs to perform back propagation through two layers, we minimize the gradient distance from two neighboring local errors. 
After the update of the $(k-1)$-th layer, we store the gradient $\delta_{\vx_{k-1}}^{(\gL_{k-1})}$, and make the output feature $\vx_{k-1}$ \emph{detached but require gradient} being the input of the $k$-th layer. At the update of $k$-th layer ($k\ge2$), we optimize with the following objective:
\begin{align}
	&\min_{\vtheta_{k}}\quad \gL_{k} + \lambda \cdot \gL_{k}^{SGR}, \label{objective}\\
	&\gL_{k}^{SGR} =  \left\| \delta_{\vx_{k-1}}^{(\gL_{k})} - \delta_{\vx_{k-1}}^{(\gL_{k-1})} \right\|_2^2 , \label{sgrloss}
\end{align}
where $\gL_{k}$ is the local error for the $k$-th layer, \emph{e.g.} the cross entropy loss, $\lambda$ is a hyper-parameter, and $\gL_{k}^{SGR} $ is an MSE between the two gradient terms. 
As shown in Figure \ref{fig1}, $\gL_{k}^{SGR} $ is added on each layer (except the first layer) to successively reconcile neighboring local updates. The implementation is outlined in Algorithm \ref{algorithm} in the Appendix.

For each layer, it is a function with two dependent variables, the input $\vx_{k-1}$ and the parameter $\vtheta_{k}$. In local learning, the gradient $\delta_{\vx_{k-1}}^{(\gL_{k})}$ cannot be passed into prior layers, and its change is decided by prior local errors, which is the cause of the discordant local updates. Our method can be interpreted as optimizing $\vtheta_{k}$ such that it not only reduces the local error $\gL_{k}$, but also enables the change of input caused by the previous layer, \emph{i.e.,} $\delta_{\vx_{k-1}}^{(\gL_{k-1})}$, to reduce $\gL_{k}$. By doing so, the local objectives are successively delivered into the last layer for a better output representation without breaking gradient isolation. Although the reconciliation concerned by our Theorem \ref{theorem} is between the local gradient and the global true gradient, while our method only deals with neighboring local gradients, the following proposition indicates that both $\left\| \bm{\epsilon} \right\|\rightarrow0$ for all local layers and $\gL_{k}^{SGR}\rightarrow0, \forall  k\ge2$, lead to an equivalence to global BP. 
\begin{proposition}
	\label{proposition1}
	For a model composed of $L$ local layers parameterized by $\vtheta_k$ with their local errors $\gL_{k}$, $k=1,...,L$, when $\gL_k^{SGR}=0$ for all layers $2\le k \le L$ for a batch of data, we have
	\begin{equation}
		\nabla_{\vtheta_k}\gL_L = \nabla_{\vtheta_k}\gL_k,\quad \forall 1\le k \le L-1, \notag
	\end{equation}
	which implies that all local updates are equivalent to learning with the global true gradient back propagated from the last-layer error $\gL_L$. 
\end{proposition}

\subsection{Local-BP Training and BP-free Training}
\label{training}

\begin{table}[t]
	\begin{center}
		\caption{Attribute comparison between our method (BP-free version) and the prior studies that also focus on the evils of BP. }
		\resizebox{\linewidth}{!}{
			\begin{tabular}{l|ccc}
				\toprule
				Method & transport free & update unlocking & BP-free \\
				\midrule 
				Globap BP & \ding{55} & \ding{55} & \ding{55} \\
				FA \cite{lillicrap2016random} & \checkmark & \ding{55} & \ding{55} \\
				DFA \cite{nokland2016direct} & \checkmark & partially & \checkmark \\
				PEPITA	{\tiny \cite{dellaferrera2022error}} & \checkmark & partially & \checkmark \\
				DRTP \cite{frenkel2021learning} & \checkmark & \checkmark & \checkmark \\
				Ours (BP-free) &  \checkmark & \checkmark & \checkmark \\
				\bottomrule
			\end{tabular}
		}
		\label{table1}
	\end{center}
\end{table}

We apply our method in both local-BP and BP-free setups. In local-BP training, a model is divided into multiple layers and each layer can be a stack of multiple blocks. Besides, the local classifier head can also have more than one layer, which is proved to be effective in improving the performance of layer-wise training \cite{belilovsky2019greedy}. Therefore, there is still back propagation within the classifier head and the backbone. We adopt the auxiliary classifier head design following \cite{belilovsky2019greedy,wang2021revisiting}, and add our $\gL_k^{(SGR)}$ on all the layers $k\ge2$. {The local error $\gL_k$ can be both the loss functions for supervised learning and self-supervised learning.}

In BP-free training, the gradient to each layer's output, $\delta\vx_{k}$, is based on an analytical update rule, and each layer is only a basic block composed of a linear transformation and a non-linear activation. Therefore, there is no back propagation through multiple layers. A comparison between several BP-free methods for supervised learning is shown in Table \ref{table1}. FA \cite{lillicrap2016random} uses a random matrix to replace the transport of weight matrix but still relies on BP. DFA \cite{nokland2016direct} directly projects the last-layer gradient into each layer via random matrices, so spares the need of BP. But local updates cannot be performed until a full forward propagation is finished, which means that update unlocking is only partially solved. PEPITA \cite{dellaferrera2022error} improves over forward-forward learning \cite{hinton2022forward} and needs a second forward propagation. As a result, update unlocking is also partially solved. 

DRTP directly acquires $\delta\vx_{k}$ by projecting labels with a random matrix. Our method in the BP-free setup also adopts a fixed matrix but with a particular structure named equiangular tight frame (ETF), which has the maximal equiangular separation
 \cite{papyan2020prevalence}, and can be formulated as,
 \begin{equation}
 	\label{etf_classifier}
 	\vm_{k_1}^T\vm_{k_2} = \left\{
 	\hspace{-4mm}
 	\begin{array}{ll}
 	&1, \quad\quad\quad\quad k_1 = k_2 \\
 	& - \frac{1}{K-1},\quad\ \ \ k_1\ne k_2
 \end{array}
 \right.
 \end{equation}
where $\vm_{k_1}$ is the classifier vector for class $k_1$ and $K$ is the number of total classes. 
It has been observed that the intermediate layers in a neural network gradually increase the separability among different classes \cite{he2023law}, and using an ETF structure as the last-layer classifier head does not harm representation learning \cite{zhu2021geometric,yang2022inducing,yang2023neural,yang2023neural-extend,zhong2023understanding,du2023problem}. Inspired by these studies, we initialize and fix an ETF classifier for each local layer to calculate the cross entropy error. It enables the gradient $\delta\vx_{k}$ to better induce separability, and our $\gL_{k}^{SGR}$ helps to carry the local separability into the last layer. 
In this way, our method keeps the benefits of DRTP but achieves significantly better performance than the methods in Table \ref{table1}, as will be shown in experiments. 

\subsection{Justification}
\label{justification}

Although our method adds a regularization on the classification loss $\gL_{k}$, which may degrade the original optimality of $\gL_{k}$ if there is only one layer, the following proposition based on a linear two-layer model shows that training with our $\gL_2^{SGR}$ loss on the second layer helps to induce a larger reduction of the classification loss $\gL_2$ in inference. 
\begin{proposition}
	\label{proposition2}
	Consider a two-layer linear model composed of $\vx_1={\vtheta_{1}}\vx_0$ and $\vx_2={\vtheta_{2}}\vx_1$, where ${\vtheta_{1}}$ and ${\vtheta_{2}}$ are the learnable linear matrices, and $\vx_0$ and $\vx_2$ are the input and output of the model, respectively. We use fixed ETF structures $\mM_1$ and $\mM_2$ for the local classifier heads and the cross entropy (CE) loss for local errors $\gL_1(\mM_1\vx_1, y)$ and $\gL_2(\mM_2\vx_2, y)$. Denote ${\gL}'_2$  as the CE loss value in inference after performing one step of gradient descent of $\gL_1$ and $\gL_2$ by Eq. (\ref{update1}) and (\ref{update2}), and denote $\hat{\gL}'_2$ as the one with our $\gL_2^{SGR}$ on the second layer. Assume that the prediction logit for the ground truth label in the second layer is larger than the one  in the first layer, we have $\hat{\gL}'_2\le {\gL}'_2$, when $\gL_2^{SGR}$ is small. 
\end{proposition}
The result empirically supports that although training with our regularization $\gL_2^{SGR}$ will shift the update direction of $\vtheta_{2}$ from the steepest descent of $\gL_2$, it helps to look for a position of $\vtheta_{2}$ such that the change of $\vx_{1}$ caused by the previous local error also enables to minimize $\gL_2$, which can lead to a larger loss reduction than only optimizing $\vtheta_{2}$.

\begin{table*}[th]  
		\caption{The results (top-1 accuracy) of local training with ResNet-18 on ImageNet. The ``Greedy'' result is our implementation following \cite{belilovsky2019greedy}. All models are divided into 4 local modules according to feature spatial resolution in ResNet-18.}
	\begin{center}
			\begin{tabular}{cc|ccccc}
				\toprule
				Greedy & \multirow{2}{*}{non-greedy} & \multicolumn{5}{c}{w/ SGR}\\
				\cite{belilovsky2019greedy} & & $\lambda=1000$  & $\lambda=2000$ & $\lambda=3000$ &  $\lambda=5000$ & $\lambda=7000$\\				
				\midrule 
				63.65 & 69.44 & 70.32 & 70.53 & 70.62 & 70.73 & 70.69\\
				\bottomrule
			\end{tabular}
		\label{ablate_lambda}
	\end{center}
	\vspace{-3mm}
\end{table*}

\begin{table}[t!]  
		\caption{The results of local training with ResNet-32 and PlainNet on CIFAR-10. Each residual or PlainNet block is a local module. The local classifier is composed of a convolution layer and a linear layer. ``reforward'' denotes forward again for each local update to use the updated output as the input of the next layer.}	
	\begin{center}
			\setlength{\tabcolsep}{10pt}
			\begin{tabular}{lcc}
				\toprule
				Method & ResNet & PlainNet\\
				\midrule
				layer-wise & 84.49$\pm$0.36 & 80.66$\pm$0.42 \\
				w/ reforward & 84.38$\pm$0.52 & 80.74$\pm$0.44 \\
				w/ SGR (ours) & \textbf{85.65$\pm$0.38} & \textbf{81.67$\pm$0.29} \\
				\bottomrule
			\end{tabular}
		\label{localbp}
	\end{center}
\end{table}

\begin{table}[t!]  
		\caption{BP-free training with PlainNet on CIFAR-10 and CIFAR-100. Each linear-nonlinear transformation in PlainNet is a local module. The local classifier is a fixed matrix as Eq. (\ref{etf_classifier}). The 2nd row results are derived from \cite{dellaferrera2022error}. }	
	\begin{center}
		\resizebox{\linewidth}{!}{
			\begin{tabular}{lccc}
				\toprule
				Method & BP-free & CIFAR-10 & CIFAR100\\
				\midrule
				FA  & \ding{55} & 57.51$\pm$0.57 & 27.15$\pm$0.53 \\
				DRTP & \checkmark & 50.53$\pm$0.81 & 20.14$\pm$0.68 \\
				PEPITA & \checkmark & 56.33$\pm$1.35 & 27.56$\pm$0.60 \\
				\midrule
				layer-wise &\checkmark & 69.17$\pm$0.91 & 48.30$\pm$0.64 \\
				w/ reforward & \checkmark& 69.05$\pm$0.63 & 47.12$\pm$0.55 \\
				w/ SGR (ours) &\checkmark & \textbf{72.40$\pm$0.75} & \textbf{49.41$\pm$0.44}\\
				\bottomrule
			\end{tabular}
		}
		\label{bpfree}
	\end{center}
\end{table}

\begin{figure}[t]
	\centering
	\includegraphics[width=1.\linewidth]{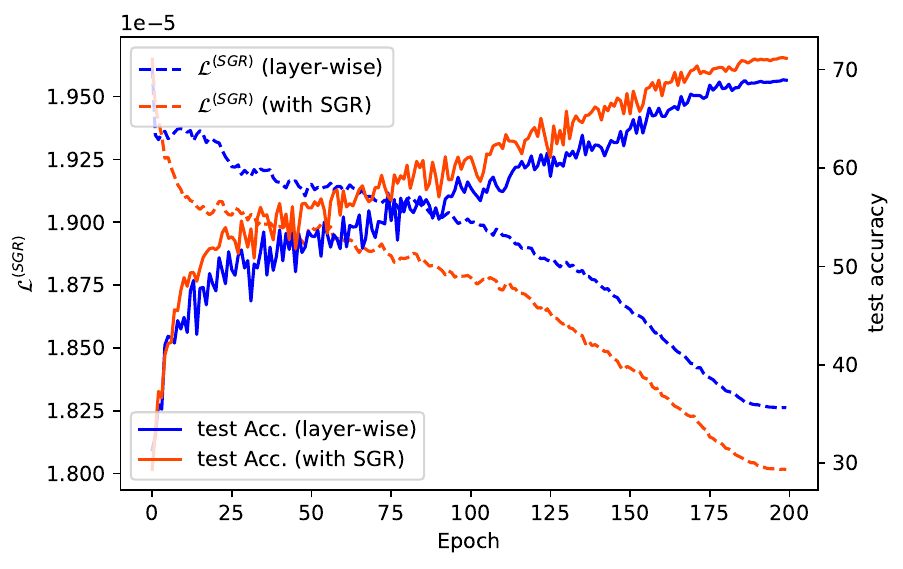}
	\vspace{-6mm}
	\caption{The training curves of average $\gL^{(SGR)}$ of all neighboring layers and accuracy on test set.}
	\label{curve}
	\vspace{-2mm}
\end{figure}

\section{Experiments}

\begin{table*}[t]  
	\renewcommand\arraystretch{1.}
		\caption{Results on ImageNet with different architectures. Experiments are conducted on 8 NVIDIA A100 GPUs. ``Train time'' is the whole training time of each experiment. ``Memory'' is the memory consumption per GPU for training tested with a batchsize of 256. Experiments of the same row are conducted with the same training setting for fair comparison. $K$ is the number of local modules.}	
	\begin{center}
			\begin{tabular}{ccccccc}
				\toprule
				 \multirow{2}{*}{Network} & \multirow{2}{*}{Method} & \multirow{2}{*}{$K$} & {top-1 Acc.} & {top-5 Acc.}  & Train time & Memory \\
				 &&& (\%) & (\%) & (hour) &  (GB)\\
				\midrule
				\multirow{5}{*}{ResNet-50} & Global BP & - & 76.55 & 93.06  & 7.0 & 21.49 \\
				& InfoPro \cite{wang2021revisiting} & 2 & 76.23 & 92.92  & 10.8 & 21.59  ($\uparrow$0.4\%)\\
				& SGR (ours) & 2 & 76.35 & 93.03  & 17.1 & 17.91 ($\downarrow$16.6\%)\\
				& InfoPro \cite{wang2021revisiting} & 3 & 75.33 & 92.57 & 21.2 & 18.73 ($\downarrow$12.8\%) \\
				& SGR (ours) & 3 & 75.42 & 92.57 & 15.1 & \textbf{15.34 ($\downarrow$28.6\%)}  \\
				\midrule      				               				
\multirow{2}{*}{ResNet-50} & Global BP & - & 71.50 & -  & 73.0 & 45.09\\
\multirow{2}{*}{(self-supervised)} & \cite{siddiqui2023blockwise} & 4 & 70.48 & -  & 140.9 & 70.53 ($\uparrow$56.4\%)\\
& SGR (ours) & 2 & 70.30 & 89.22 & 98.6 & \textbf{34.86 ($\downarrow$22.7\%)} \\    					
				\midrule
				\multirow{5}{*}{ResNet-101} & Global BP & - & 77.97 & 94.06  & 11.0 & 31.70\\
                & InfoPro \cite{wang2021revisiting} & 2 & 77.61 & 93.78 & 18.6 & 26.14 ($\downarrow$17.6\%)\\
                & SGR (ours) & 2 & 77.69 & 93.84  & 28.5 & 22.99 ($\downarrow$27.5\%)\\
				& InfoPro \cite{wang2021revisiting} & 3 & 77.02 & 93.47  & 24.3 & 21.87 ($\downarrow$31.0\%)\\
				& SGR (ours) & 3 & 77.02 & 93.24  & 22.9 & \textbf{18.18 ($\downarrow$42.6\%)}\\                
				\midrule
				\multirow{3}{*}{VIT-small} & Global BP & - & 79.40 & 94.11  & 52.2& 20.70 \\
				& InfoPro \cite{wang2021revisiting} & 3 & 78.15 & 94.00 & 61.8 & 14.96 ($\downarrow$27.7\%)\\
				& SGR (ours) & 3 & 78.65 & 94.03  & 62.5 & \textbf{11.73 ($\downarrow$43.3\%)}\\ 
				\midrule
				\multirow{3}{*}{Swin-tiny} & Global BP & - & 81.18 & 95.61  & 43.8 & 26.78\\
				& InfoPro \cite{wang2021revisiting} & 2 & 79.38 & 94.14  & 53.2 & 23.34 ($\downarrow$12.9\%)\\
				& SGR (ours) & 3 & 80.05 & 94.95  & 56.3 & \textbf{15.83 ($\downarrow$40.9\%)}\\				
				\midrule      				               				
				\multirow{3}{*}{Swin-small} & Global BP & - & 83.02 & 96.29  & 75.9 & 42.60\\
				& InfoPro \cite{wang2021revisiting} & 2 & 81.19 & 95.00  & 86.6 & 33.08 ($\downarrow$22.4\%)\\
& SGR (ours) & 3 & 81.64 & 95.45 & 90.8 & \textbf{22.83  ($\downarrow$46.4\%)}\\    							  					
				\bottomrule
			\end{tabular}
		\label{imagenet_result}
	\end{center}
	\vspace{-2mm}
\end{table*}

In experiments, we test the effectiveness of our proposed successive gradient reconciliation (SGR), and demonstrate that it is able to significantly save memory consumption of CNN and Transformer architectures, and surpass previous studies on CIFAR-10, CIFAR-100, and ImageNet datasets. The implementation details are provided in Appendix \ref{training_details}.

\subsection{Ablation Studies}

We ablate the hyper-parameter $\lambda$ in Eq. (\ref{objective}) and integrate our SGR on non-greedy layer-wise training to show its ability to improve performance in both local-BP and BP-free cases. 

We add our $\gL^{(SGR)}$ as defined in Eq. (\ref{sgrloss}) with different $\lambda$ in layer-wise training of ResNet-18 \cite{he2016deep} on ImageNet. As shown in Table \ref{ablate_lambda}, most choices of $\lambda$ achieve more than 1\% performance improvement over non-greedy layer-wise training. 
These results with different $\lambda$ do not have a large deviation, which indicates that our method has stable performance and is not sensitive to the choice of $\lambda$.
As $\lambda$ increases from 1000 to 5000, the accuracy slightly gets better, showing that a stronger reconciliation  of local updates within a proper range can lead to a better performance. 

In local-BP training, we use ResNet-32 and PlainNet that removes the identity connection of ResNet. Each local module may contain multiple layers and the local classifier is composed of a convolution layer and a linear layer, so there is still BP within each local update. As shown in Table \ref{localbp}, when armed with our method, layer-wise training attains more than 1\% accuracy improvement with both ResNet and PlainNet. The reconciliation of local updates can be also naively achieved by a second forward propagation after the local update of each layer such that the input of the next layer is based on the updated output. As shown in Tables \ref{localbp} and \ref{bpfree}, this practice does not bring obvious performance gain because it performs too many steps of optimization for each batch of data and thus overfits in each iteration.

In BP-free training, we use PlainNet where each block is only a linear-nonlinear transformation and the local classifier is a fixed ETF structure as defined in Eq. (\ref{etf_classifier}). Therefore, the gradient \emph{w.r.t.} output feature of each block can be analytically derived and there is no BP within each local module. As shown in Table \ref{bpfree}, our method improves by more than 3\% on CIFAR-10 and more than 1\% on CIFAR-100. As shown in Figure \ref{curve}, when our method is adopted, the average loss value of $\gL^{(SGR)}$ is significantly lower than the baseline without our method, and accordingly, the test accuracy is better throughout training. It reveals that gradient reconciliation among local updates, which our method targets, is correlated with the performance of local learning.

\begin{table*}[!t]  
	\begin{center}
		\caption{Local learning results with ResNet-32 on CIFAR-10 and CIFAR-100 and comparison with prior studies. $K$ is the number of local modules divided from ResNet-32. $K$=16 refers to the case where each residual block including the stem layer correspond to one local module. $K$=3 divides the model according to feature spatial resolution. $K=2$ leaves the last 5 blocks as the second local module.}
		\resizebox{\linewidth}{!}{
			\begin{tabular}{lcccccc}
				\toprule
				\multirow{2}{*}{Method} & \multicolumn{3}{c}{CIFAR-10 (BP: 92.82$\pm$0.22)} & \multicolumn{3}{c}{CIFAR-100 (BP: 74.78$\pm$0.31)} \\
				& $K$=16 & $K$=3 & $K$=2 &  $K$=16 & $K$=3 & $K$=2 \\
				\midrule
				Greedy \cite{belilovsky2019greedy} & 76.96$\pm$0.54 & 85.12$\pm$0.35 & 90.06$\pm$0.43  & 48.77$\pm$0.41 & 63.55$\pm$0.32 & 70,68$\pm$0.38  \\
				DGL \cite{belilovsky2020decoupled} & 84.01$\pm$0.40 & 87.61$\pm$0.51 & 91.05$\pm$0.27 & 63.31$\pm$0.28 & 70,59$\pm$0.39 & 72.34$\pm$0.31 \\
				InfoPro \cite{wang2021revisiting}  & \textbf{85.69$\pm$0.47} & 90.43$\pm$0.36 & 91.74$\pm$0.29 & 66.27$\pm$0.44 & 71.44$\pm$0.24 & 73.59$\pm$0.25 \\
				SGR (ours) & 85.65$\pm$0.38  & \textbf{91.34$\pm$0.45} & \textbf{92.91$\pm$0.36} & \textbf{66.61$\pm$0.31} & \textbf{72.15$\pm$0.27} & \textbf{74.63$\pm$0.20} \\
				\bottomrule
			\end{tabular}
		}
		\label{cifar_result}
	\end{center}
\end{table*}

\begin{table}[t!]  
	\vspace{-2mm}
	\caption{Results of VGG-Net with different depth on CIFAR-10.}	
	\begin{center}
		\resizebox{\linewidth}{!}{
			\setlength{\tabcolsep}{10pt}
			\begin{tabular}{lcccccc}
				\toprule
				Method & $d$=12 & $d$=17 & $d$=27 & $d$=37 & $d$=47 & $d$=57\\
				\midrule
				Global BP & 90.76 & 90.95 & 90.46 & 78.47 & 37.68 & fail\\
				SGR (ours) & 84.81 & 86.38 & 86.95 & 86.09 & 85.30 & 85.51\\
				\bottomrule
			\end{tabular}
			}
		\label{vgg_depth}
	\end{center}
	\vspace{-2mm}
\end{table}

\begin{figure*}[t]
	\begin{subfigure}{0.33\textwidth}
		\centering
		\includegraphics[width=1.\linewidth]{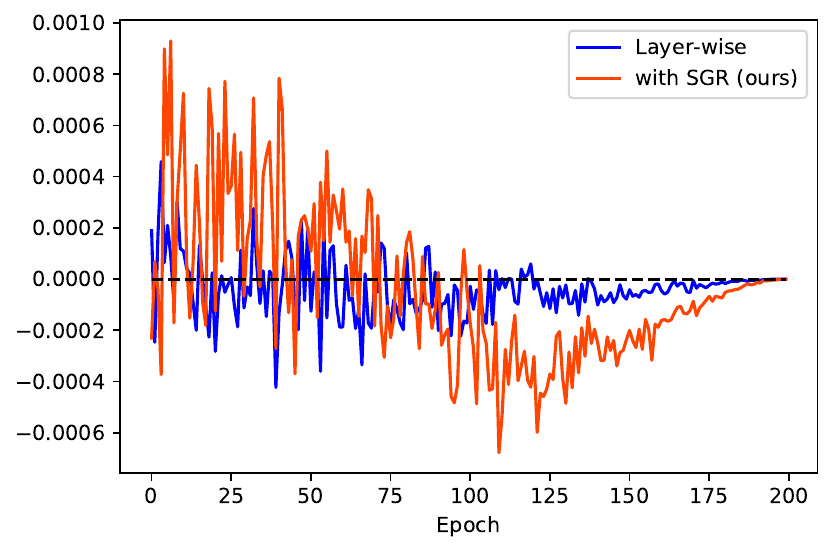}  
		\caption{Layer-2}
		\label{fig:layer2}
	\end{subfigure}
	\begin{subfigure}{0.33\textwidth}
		\centering
		\includegraphics[width=1.\linewidth]{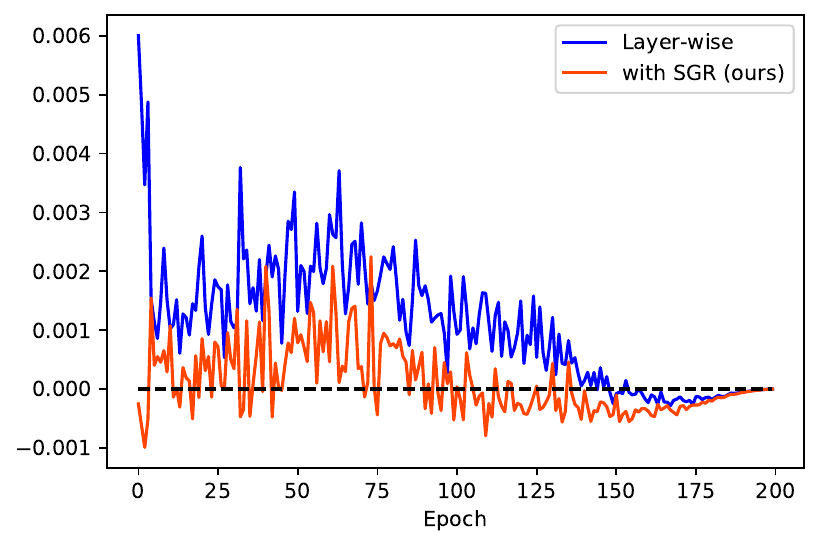}  
		\caption{Layer-3}
		\label{fig:layer3}
	\end{subfigure}
	\begin{subfigure}{0.33\textwidth}
		\centering
		\includegraphics[width=1.\linewidth]{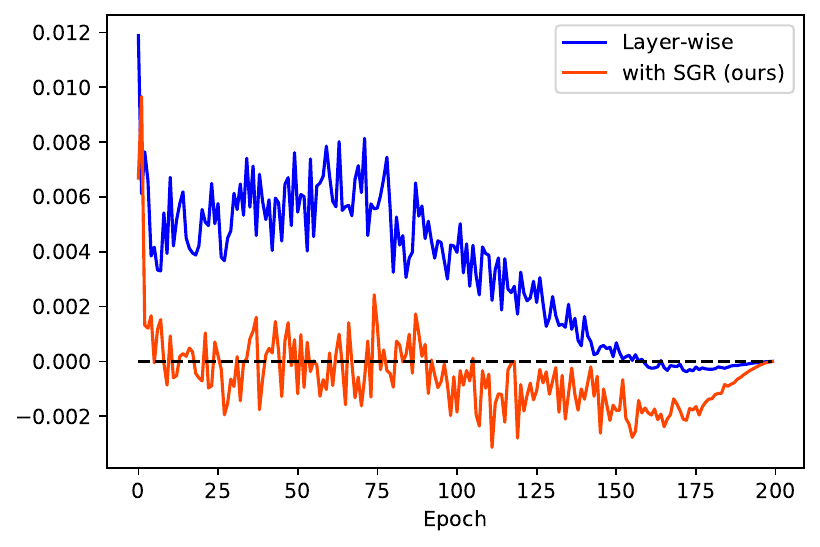}  
		\caption{Layer-4}
		\label{fig:layer4}
	\end{subfigure}
	\vspace{-5mm}
	\caption{We measure the change of classification loss in each layer caused by the input update from prior layers using a 4-layer PlainNet. The black dashed line denotes the zero baseline such that the area below it means that the updates of prior layers can produce a new output feature enabling to reduce the loss value of the current layer as its input. }
	\vspace{-3mm}
	\label{fig:diff_layer}
\end{figure*}

\subsection{Results on ImageNet}

Our method is able to help CNN and Transformer \cite{dosovitskiy2020image,liu2021swin} architectures save a significant proportion of memory consumption on ImageNet, while achieving competitive performance with global BP. As shown in Table \ref{imagenet_result}, our method on ResNet-50 has better performances than \cite{wang2021revisiting} and a similar performance to \cite{siddiqui2023blockwise} with larger memory conservation than both methods in supervised and self-supervised learning, respectively. When $K=3$, our method trains faster with lower memory consumption than InfoPro on both ResNet-50 and ResNet-101. On the three Transformer-based architectures, our method saves more than 40\% memory without inducing a large increment of train time or unbearable performance loss (more than 1.5\%) compared with global BP, and achieves better performances than InfoPro with close train time overhead.

\subsection{Results on CIFAR}

We verify the performance of our method using ResNet-32 on CIFAR-10 and CIFAR-100 and compare with greedy layer-wise training \cite{belilovsky2019greedy}, DGL \cite{belilovsky2020decoupled}, and InfoPro \cite{wang2021revisiting}. Following \cite{belilovsky2019greedy}, we adopt a local classifier composed of a convolution layer and a linear layer for classification. As shown in Table \ref{cifar_result}, our SGR surpasses all compared methods in most cases. Especially when $K$=2, we achieve a competitive performance with global BP (our 92.91\% v.s. BP 92.82\% on CIFAR-10, and our 74.63\% v.s. BP 74.78\% on CIFAR-100). 

Since local learning spares the effort of a global BP, it will not suffer from the notorious gradient vanishing and explosion problems that are easy to emerge in training deep neural networks such as the VGG architecture \cite{simonyan2014very}. As shown in Table \ref{vgg_depth}, we compare the performances of our method and training with global BP in VGG-Net of different depth. As the depth goes larger, the performance of global BP decreases sharply until a failed convergence. In contrast, our SGR is able to keep a stable performance even in a very deep network.

\subsection{Analysis}

In traditional local learning, local updates are not reconciled, so the updates of prior layers do not necessarily change the output feature towards a direction that minimizes the current local loss. In this subsection, we investigate whether our SGR helps to remedy this defect. We measure the classification loss change in one layer caused by the update of its input feature, \emph{i.e.,} $\Delta \gL_k=\gL_k(\vx_{k-1}^{(t+1)})-\gL_k(\vx_{k-1}^{(t)})$, where $\vx_{k-1}^{(t)}$ is the input feature of this local layer in the $t$-th iteration and $\vx_{k-1}^{(t+1)}$ is the new one after the update of all its previous layers $1,..,k-1$.

As shown in Figure \ref{fig:diff_layer}, for a 4-layer PlainNet, $\Delta \gL_2$ in the 2nd layer with and without our method are both surrounding zero. In deeper layers, the scale of $\Delta \gL$ without our method (the blue one) is growing larger, which indicates that the discordant local updates have an accumulative effect on deep layers. Besides, the curves without our method in the 3-rd and 4-th layers are above  zero most of the time, implying that the prior local updates contribute little to the deep layers' learning. As a comparison, our method does not grow the scale of $\Delta \gL$ apparently in deep layers. Particularly in the 4-th layer, the curve with our method is below zero in most epochs, which means the first 3 layers help to produce a better feature as the input of the last layer, and is in line with our better performance observed in Figure \ref{curve}.

\section{Conclusion}

In this paper, we point out a fundamental defect of local learning that global BP is naturally immune to. Our theoretical result indicates that the convergence of local learning cannot be assured when the local updates are not reconciled. Based on the result, we propose a method, named successive gradient reconciliation, to successively reconcile local updates without breaking gradient isolation or introducing any learnable parameters. Our method can be applied to both local-BP and BP-free local learning. Experimental results demonstrate that our method surpasses previous methods and is able to achieve a comparable performance with global BP saving more than 40\% memory consumption for CNN and Transformer architectures. Future studies may explore applying our method to large model finetuning. 

\section*{Impact Statement}


Our method reduces memory consumption in training neural network, making machine learning models more friendly in resource-limited environments. Furthermore, by aligning more closely with biological learning processes, our method may contribute to the intersection of AI and neuroscience, potentially leading to developments of more biologically plausible AI systems. The main goal of this paper is to advance the field of machine learning. There is no ethic impact that we feel must be specifically highlighted here.

\section*{Acknowledgements}

This work was supported by KAUST-Oxford CRG Grant Number DFR07910, the King Abdullah University of Science and Technology (KAUST) Office of Sponsored Research (OSR) under Award No. OSR-CRG2021-4648, SDAIA-KAUST Center of Excellence in Data Science, Artificial Intelligence (SDAIA-KAUST AI), and a UKRI grant Turing AI Fellowship (EP/W002981/1).



\bibliography{BPE}
\bibliographystyle{icml2024}

\newpage
\appendix
\onecolumn
\section{Proof of Theorem \ref{theorem}}

We restate the assumptions and results, and then provide the proof. 

\begin{assumption}[PL Condition]
	\label{app:assump1}
	Let $\gL_2^*$ be the optimal function value of the second layer loss $\gL_2$. There exists a $\mu$ such that $\forall \vtheta_1, \vtheta_2$, we have: 
	\begin{align}
		\left\| \nabla \gL_2(\vtheta_1, \vtheta_2) \right\|^2&=\left\| \nabla_{\vtheta_1} \gL_2(\vtheta_1, \vtheta_2) \right\|^2+\left\| \nabla_{\vtheta_2} \gL_2(\vtheta_1, \vtheta_2) \right\|^2 \notag \\
		&\ge \mu\left(  \gL_2(\vtheta_1, \vtheta_2) - \gL_2^* \right), \notag
	\end{align}
	where $\nabla\gL_2=[\nabla_{\vtheta_1}\gL_2; \nabla_{\vtheta_2}\gL_2]$.
\end{assumption}

\begin{assumption}[Layer-wise Lipschitz and global Lipschitz]
	\label{app:assump2}
	There exists $L_1$, $L_2$, and $L > 0$ such that for all $\vtheta_{1,a}, \vtheta_{1,b}, \vtheta_1, \vtheta_{2,a}, \vtheta_{2,b}, \vtheta_2$, we have:
	\begin{align}
		&\left\| \nabla_{\vtheta_1} \gL_2(\vtheta_{1,a}, \vtheta_2) - \nabla_{\vtheta_1} \gL_2(\vtheta_{1,b}, \vtheta_2) \right\| \le L_1 \left\| \vtheta_{1,a} - \vtheta_{1,b} \right\|,  \notag\\
		&\left\| \nabla_{\vtheta_2} \gL_2(\vtheta_{1}, \vtheta_{2,a}) - \nabla_{\vtheta_2} \gL_2(\vtheta_{1}, \vtheta_{2,b}) \right\| \le L_2 \left\| \vtheta_{2,a} - \vtheta_{2,b} \right\| \notag,
	\end{align}	
	and
	\vspace{-1mm}
	\begin{align}
		\left\| \nabla \gL_2(\vtheta_{1,a}, \vtheta_{2,a}) - \nabla \gL_2(\vtheta_{1,b}, \vtheta_{2,b}) \right\| \le L \left\|
		\begin{bmatrix}
			\vtheta_{1,a} - \vtheta_{1,b} \\
			\vtheta_{2,a} - \vtheta_{2,b}
		\end{bmatrix} 
		\right\|. \notag
	\end{align}
\end{assumption}

\begin{theorem}
	\label{app:theorem}
	Based on Assumptions \ref{assump1} and \ref{assump2}, if the learning rates are set as $\eta_1^{(i)}=\eta_1$ and $\eta_2^{(i)}=\eta_2$, where
	\begin{equation}
		\left\{\hspace{-4mm}
		\begin{array}{cl}
			& 0<\eta_1\le \min \left( \frac{\sqrt{L_1^2+8L^2}-L_1}{4L^2}, \frac{2}{\mu}, \frac{1}{2L_2} \right)\\
			& \max \left(0, \frac{1-\sqrt{1-2L_2\eta_1}}{L_2} \right) < \eta_2 \le \frac{1+\sqrt{1-2L_2\eta_1}}{L_2} \notag
		\end{array}
		\right.
		,
	\end{equation}
	we have the following convergence
	\begin{equation}
		\label{app:convergence1}
		\gL_2^{(i+1,i+1)}-\gL_2^* \le \left(1-\alpha\mu\right)\left(\gL_2^{(i,i)}-\gL_2^*\right)
		+ \alpha \left\| \bm{\epsilon}^{(i)} \right\|^2,
	\end{equation}
	and recursively applying Eq. (\ref{app:convergence1}) we have
	\begin{align}
		\label{app:convergence2}
		\gL_2^{(i+1,i+1)}-\gL_2^* \le& \left(1-\alpha\mu\right)^{i+1}\left(\gL_2^{(0,0)}-\gL_2^*\right) \notag \\
		&+ \alpha \sum_{k=0}^{i}(1-\alpha\mu)^k\left\| \bm{\epsilon}^{(i-k)} \right\|^2,
	\end{align}
	where $\gL_2^{(i,i)}$ denotes $\gL_2(\vtheta_1^{(i)}, \vtheta_2^{(i)})$, \emph{i.e.,} the second layer loss value with parameters $\vtheta_1^{(i)}$ and $\vtheta_2^{(i)}$ in the $i$-th iteration, $\alpha=\frac{\eta_1}{2}$, and $\bm{\epsilon}^{(i)} \triangleq \nabla_{\vtheta_1}\gL_2^{(i,i)} - \nabla_{\vtheta_1}\gL_1^{(i)}$.  
\end{theorem}

\begin{lemma}
	\label{lemma1}
	With assumption \ref{app:assump1}, we have 
	\begin{equation}
		\gL_2^{(i,i)} - \gL_2^{(i+1,i+1)} \ge \eta_2^{(i)} \left(1-\frac{L_2}{2}\eta_2^{(i)} \right)\left\|\nabla_2 \gL_2^{(i+1,i)} \right\|^2 + \eta_1^{(i)}\nabla_1^T\gL_1^{(i)}\bm{\epsilon}^{(i)} + \eta_1^{(i)}\left(1-\frac{L_1}{2} \eta_1^{(i)} \right)\left\| \nabla_1\gL_1^{(i)}\right\|^2,
	\end{equation}
	where $\bm{\epsilon}^{(i)}=\nabla_1\gL_2^{(i,i)} - \nabla_1\gL_1^{(i)}$, $\nabla_1\gL_2^{(i,i)}$ is the abbreviation of $\nabla_{\vtheta_1}\gL_2(\vtheta_1^{(i)}, \vtheta_2^{(i)})$, and $\nabla_1\gL_1^{(i)}$ refers to $\nabla_{\vtheta_1}\gL_1(\vtheta_1^{(i)})$.
\end{lemma}

\begin{lemma}
	\label{lemma2}
	With assumption \ref{app:assump1}, we have 
	\begin{equation}
		\left\| \nabla \gL_2^{(i,i)} \right\|^2 - \left\| \bm{\epsilon}^{(i)}\right\|^2\le \left( 2\left(L\eta_1^{(i)} \right)^2+1 \right)\left\| \nabla_1\gL_1^{(i)} \right\|^2 + 2\left\| \nabla_2\gL_2^{(i+1,i)}\right\|^2 + 2\nabla_1^T\gL_1^{(i)}\bm{\epsilon}^{(i)},
	\end{equation}
	where $\bm{\epsilon}^{(i)}=\nabla_1\gL_2^{(i,i)} - \nabla_1\gL_1^{(i)}$, and $\nabla\gL_2^{(i,i)}=[\nabla_{\vtheta_1}\gL_2^{(i,i)}; \nabla_{\vtheta_2}\gL_2^{(i,i)}]$.
\end{lemma}

\begin{proof}[Proof of Lemma \ref{lemma1}]
	The layer wise training based on Eq. (\ref{update1}) and (\ref{update2}) is restated as follows:
	\begin{align}
		&\vtheta_1^{(i+1)} \leftarrow \vtheta_1^{(i)} -\eta_1^{(i)} \nabla_{\vtheta_1}\gL_1(\vtheta_1^{(i)}), \notag\\
		&\vtheta_2^{(i+1)} \leftarrow \vtheta_2^{(i)} -\eta_2^{(i)} \nabla_{\vtheta_2}\gL_2(\vtheta_1^{(i+1)}, \vtheta_2^{(i)}). \notag
	\end{align}
	
		Based on descent lemma \cite{beck2013convergence} and Assumption \ref{app:assump2}, we have
  \begin{equation}\label{A}
	\begin{aligned}
		\mathcal{L}_{2}^{(i+1, i+1)} 
		& \le 
		\mathcal{L}_{2}^{(i+1, i)}
		+ 
		\nabla_{2}^{T} \mathcal{L}_{2}^{(i+1, i)}
		\left(
		\bm{\theta}_{2}^{(i+1)} - \bm{\theta}_{2}^{(i)}
		\right)
		+ 
		\frac{L_2}{2} \Big\Vert \bm{\theta}_{2}^{(i+1)} - \bm{\theta}_{2}^{(i)} \Big\Vert^{2}
		\\ & = 
		\mathcal{L}_{2}^{(i+1, i)} - \eta_2^{(i)} \left( 1 - \frac{L_2}{2} \eta_2^{(i)} \right)
		\Big\Vert \nabla_{2} \mathcal{L}_{2}^{(i+1, i)} \Big\Vert^{2},
	\end{aligned}
\end{equation} 
and
\begin{equation}\label{B}
	\begin{aligned}
		\mathcal{L}_{2}^{(i+1, i)}
		& \le 
		\mathcal{L}_{2}^{(i, i)}
		+ 
		\nabla_{1}^{T} \mathcal{L}_{2}^{(i, i)}
		\left(
		\bm{\theta}_{1}^{(i+1)} - \bm{\theta}_{1}^{(i)}
		\right)
		+ 
		\frac{L_1}{2} \Big\Vert \bm{\theta}_{1}^{(i+1)} - \bm{\theta}_{1}^{(i)} \Big\Vert^{2}
		\\ & = 
		\mathcal{L}_{2}^{(i, i)} - 
		\eta_1^{(i)} \nabla_{1}^{T} \mathcal{L}_{2}^{(i, i)} \nabla_1 \mathcal{L}_{1}^{(i)}
		+ \frac{L_1}{2}(\eta_1^{(i)})^{2} \Big\Vert \nabla_1 \mathcal{L}_{1}^{(i)} \Big\Vert^{2}.
	\end{aligned}
\end{equation}
Sum up (\ref{A}) and (\ref{B}) to obtain
\begin{equation*}
	\begin{aligned}
		\mathcal{L}_{2}^{(i, i)} - \mathcal{L}_{2}^{(i+1, i+1)}
		& \ge 
		\eta_2^{(i)} \left( 1 - \frac{L_2}{2} \eta_2^{(i)} \right)
		\Big\Vert \nabla_{2} \mathcal{L}_{2}^{(i+1, i)} \Big\Vert^{2}
		+
		\eta_1^{(i)} \nabla_{1}^{T} \mathcal{L}_{2}^{(i, i)} \nabla_1 \mathcal{L}_{1}^{(i)}
		- \frac{L_1}{2}(\eta_1^{(i)})^{2} \Big\Vert \nabla_1 \mathcal{L}_{1}^{(i)} \Big\Vert^{2}.
	\end{aligned}
\end{equation*}
Then let $\nabla_1 \mathcal{L}_{2}^{(i, i)} = \bm{\epsilon}^{i} + \nabla_1 \mathcal{L}_1^{(i)}$:
\begin{equation*}
	\begin{aligned}
		\mathcal{L}_{2}^{(i, i)} - \mathcal{L}_{2}^{(i+1, i+1)}
		& \ge 
		\eta_2^{(i)} \left( 1 - \frac{L_2}{2} \eta_2^{(i)} \right)
		\Big\Vert \nabla_{2} \mathcal{L}_{2}^{(i+1, i)} \Big\Vert^{2}
		+
		\eta_1^{(i)} \nabla_1^{T} \mathcal{L}_{1}^{(i)} \bm{\epsilon}^{i}
		+ 
		\eta_1^{(i)}
		\left(
		1 -
		\frac{L_1}{2}(\eta_1^{(i)})
		\right)
		\Big\Vert \nabla_1 \mathcal{L}_{1}^{(i)} \Big\Vert^{2},
	\end{aligned}
\end{equation*}
which concludes the proof of Lemma \ref{lemma1}.
\end{proof}

\begin{proof}[Proof of Lemma \ref{lemma2}]
  \begin{equation*}
	\begin{aligned}
		\Big\Vert \nabla \mathcal{L}_{2}^{(i, i)} \Big\Vert^{2} & = 
		\Big\Vert \nabla_2 \mathcal{L}_{2}^{(i, i)} \Big\Vert^{2} + \Big\Vert \nabla_1 \mathcal{L}_{2}^{(i, i)} \Big\Vert^{2}
		\\ & \le 
		\left(
		\Big\Vert \nabla_2 \mathcal{L}_{2}^{(i, i)} - \nabla_2 \mathcal{L}_2^{(i+1, i)} \Big\Vert + 
		\Big\Vert \nabla_2 \mathcal{L}_2^{(i+1, i)} \Big\Vert
		\right)^{2}
		+
		\Big\Vert \nabla_1 \mathcal{L}_{2}^{(i, i)} \Big\Vert^{2}
		\\ & =
		\Big\Vert \nabla_2 \mathcal{L}_{2}^{(i, i)} - \nabla_2 \mathcal{L}_2^{(i+1, i)} \Big\Vert^2 + \Big\Vert \nabla_2 \mathcal{L}_2^{(i+1, i)} \Big\Vert^2
		+2 \Big\Vert \nabla_2 \mathcal{L}_{2}^{(i, i)} - \nabla_2 \mathcal{L}_2^{(i+1, i)} \Big\Vert  \Big\Vert \nabla_2 \mathcal{L}_2^{(i+1, i)} \Big\Vert^2
		+	\Big\Vert \nabla_1 \mathcal{L}_{2}^{(i, i)} \Big\Vert^{2} 
		\\ &\le 
		2\left(
		\Big\Vert \nabla_2 \mathcal{L}_{2}^{(i, i)} - \nabla_2 \mathcal{L}_2^{(i+1, i)} \Big\Vert^{2} + 
		\Big\Vert \nabla_2 \mathcal{L}_2^{(i+1, i)} \Big\Vert^{2}
		\right)
		+
		\Big\Vert \nabla_1 \mathcal{L}_{2}^{(i, i)} \Big\Vert^{2}
		\\ & \le
		2\left(
		\Big\Vert \nabla_2 \mathcal{L}_{2}^{(i, i)} - \nabla_2 \mathcal{L}_2^{(i+1, i)} \Big\Vert^{2} + 
		\Big\Vert \nabla_1 \mathcal{L}_{2}^{(i, i)} - \nabla_1 \mathcal{L}_2^{(i+1, i)} \Big\Vert^{2} +
		\Big\Vert \nabla_2 \mathcal{L}_2^{(i+1, i)} \Big\Vert^{2}
		\right)
		+
		\Big\Vert \nabla_1 \mathcal{L}_{2}^{(i, i)} \Big\Vert^{2}
		\\ & =
		2\left(
		\Big\Vert \nabla \mathcal{L}_{2}^{(i, i)} - \nabla \mathcal{L}_2^{(i+1, i)} \Big\Vert^{2} +
		\Big\Vert \nabla_2 \mathcal{L}_2^{(i+1, i)} \Big\Vert^{2}
		\right)
		+
		\Big\Vert \nabla_1 \mathcal{L}_{2}^{(i, i)} \Big\Vert^{2}
		\\ & \le
		2L^{2}
		\Big\Vert \bm{\theta}_1^{(i)} - \bm{\theta}_1^{(i+1)} \Big\Vert^{2}
		+ 
		2\Big\Vert \nabla_2 \mathcal{L}_2^{(i+1, i)} \Big\Vert^{2}
		+
		\Big\Vert \nabla_1 \mathcal{L}_{2}^{(i, i)} \Big\Vert^{2}
		\\ & \le
		2\left(
		L \eta_1^{(i)}
		\right)
		^{2}
		\Big\Vert \nabla_1 \mathcal{L}_1^{(i)} \Big\Vert^{2}
		+ 
		2\Big\Vert \nabla_2 \mathcal{L}_2^{(i+1, i)} \Big\Vert^{2}
		+
		\Big\Vert \nabla_1 \mathcal{L}_{2}^{(i, i)} \Big\Vert^{2}.
	\end{aligned}
\end{equation*}
Since $\nabla_1 \mathcal{L}_{2}^{(i, i)} = \bm{\epsilon}^{(i)} + \nabla_1 \mathcal{L}_1^{(i)}$, we have
\begin{equation*}
	\begin{aligned}
		& \Big\Vert \nabla \mathcal{L}_{2}^{(i, i)} \Big\Vert^{2} 
		\\ \le &
		2\left(
		L \eta_1^{(i)}
		\right)
		^{2}
		\Big\Vert \nabla_1 \mathcal{L}_1^{(i)} \Big\Vert^{2}
		+ 
		2\Big\Vert \nabla_2 \mathcal{L}_2^{(i+1, i)} \Big\Vert^{2}
		+
		\Big\Vert \nabla_1 \mathcal{L}_{2}^{(i, i)} - \nabla_1 \mathcal{L}_1^{(i)} + \nabla_1 \mathcal{L}_1^{(i)} \Big\Vert^{2}
		\\ = &
		2\left(
		L \eta_1^{(i)}
		\right)
		^{2}
		\Big\Vert \nabla_1 \mathcal{L}_1^{(i)} \Big\Vert^{2}
		+ 
		2\Big\Vert \nabla_2 \mathcal{L}_2^{(i+1, i)} \Big\Vert^{2}
		+
		\Big\Vert \nabla_1 \mathcal{L}_{2}^{(i, i)} - \nabla_1 \mathcal{L}_1^{(i)} \Big\Vert^{2}
		+
		\Big\Vert \nabla_1 \mathcal{L}_1^{(i)} \Big\Vert^{2}
		+
		2 \left(
		\nabla_1 \mathcal{L}_{2}^{i, i} - \nabla_1 \mathcal{L}_1^{(i)}
		\right)^{T} 
		\nabla_1 \mathcal{L}_1^{(i)}
		\\ = &
		\left(  2\left( L \eta_1^{(i)}  \right)^{2} + 1\right)
		\Big\Vert \nabla_1 \mathcal{L}_1^{(i)} \Big\Vert^{2}
		+ 2\Big\Vert \nabla_2 \mathcal{L}_2^{(i+1, i)} \Big\Vert^{2}
		+ \Big\Vert \bm{\epsilon}^{(i)} \Big\Vert^{2} + 
		2 \nabla_1^{T} \mathcal{L}_1^{(i)} \bm{\epsilon}^{(i)}.
	\end{aligned}
\end{equation*}
Therefore, we have,
\begin{equation}
		\Big\Vert \nabla \mathcal{L}_{2}^{(i, i)} \Big\Vert^{2} - \Big\Vert \bm{\epsilon}^{(i)} \Big\Vert^{2} 
		\le	\left(  2\left( L \eta_1^{(i)}  \right)^{2} + 1\right)
	\Big\Vert \nabla_1 \mathcal{L}_1^{(i)} \Big\Vert^{2}
	+ 2\Big\Vert \nabla_2 \mathcal{L}_2^{(i+1, i)} \Big\Vert^{2}
	+ 
	2 \nabla_1^{T} \mathcal{L}_1^{(i)} \bm{\epsilon}^{(i)},
\end{equation}
which concludes the proof of Lemma \ref{lemma2}.
\end{proof}

\begin{proof}[Proof of Theorem \ref{app:theorem}]
 
We let $\eta_1^{(i)}=\eta_1$ and $\eta_2^{(i)}=\eta_2$. Recall Lemma (\ref{lemma1}) and (\ref{lemma2}), we could calculate a factor $\alpha$ such that 
\begin{equation}\label{conditionB}
	\begin{aligned}
		\mathcal{L}_{2}^{(i, i)} - \mathcal{L}_{2}^{(i+1, i+1)} & \ge 
		\left[\begin{array}{ccccccc}
			\eta_2 \left( 1 - \frac{L_2}{2} \eta_2 \right) \\
			\eta_1
			\left(
			1 -
			\frac{L_1}{2}(\eta_1)
			\right) \\
			\eta_1 \\
		\end{array}\right]^{T}
		\left[\begin{array}{ccccccc}
			\Big\Vert \nabla_{2} \mathcal{L}_{2}^{(i+1, i)} \Big\Vert^{2} \\
			\Big\Vert \nabla_1 \mathcal{L}_{1}^{(i)} \Big\Vert^{2} \\
			\nabla_1^{T} \mathcal{L}_{1}^{(i)} \bm{\epsilon}^{i} \\
		\end{array}\right]
		\\ & \overset{\textbf{(A)}}{\ge}
		\alpha
		\left[\begin{array}{ccccccc}
			2 \\
			2\left( L \eta_1  \right)^{2} + 1 \\
			2 \\
		\end{array}\right]^{T}
		\left[\begin{array}{ccccccc}
			\Big\Vert \nabla_{2} \mathcal{L}_{2}^{(i+1, i)} \Big\Vert^{2} \\
			\Big\Vert \nabla_1 \mathcal{L}_{1}^{(i)} \Big\Vert^{2} \\
			\nabla_1^{T} \mathcal{L}_{1}^{(i)} \bm{\epsilon}^{i} \\
		\end{array}\right]
		\\ & \overset{\textbf{(B)}}{\ge} \alpha \left(
		\Big\Vert \nabla \mathcal{L}_{2}^{(i, i)} \Big\Vert^{2} - \Vert \bm{\epsilon}^{(i)} \Vert^{2}
		\right),
	\end{aligned}
\end{equation}
where \textbf{(B)} holds because of the conclusion of Lemma \ref{lemma2}, and 
to ensure the correctness of \textbf{(A)}, we solve the following inequalities
\begin{equation}\label{eqs}
	\left\{
	\begin{aligned}
		& \eta_2 \left( 1 - \frac{L_2}{2} \eta_2 \right) \ge 2 \alpha \\
		& \eta_1 \left( 1 - \frac{L_1}{2}(\eta_1) \right) \ge 2 \alpha \left( L \eta_1  \right)^{2} + \alpha\\
		& \eta_1 = 2 \alpha. \\
	\end{aligned}
	\right.
\end{equation}
The third of Eq.(\ref{eqs}) implies $\alpha = \frac{\eta_1}{2}$. 
Plug it into the second to obtain $-\frac{\sqrt{L_1^{2} + 8 L^2} + L_1}{4 L^{2}} < \eta_1 \le \frac{\sqrt{L_1^{2} + 8 L^2} - L_1}{4 L^{2}}$. 
Plug it into the first one to obtain $\frac{1 - \sqrt{1-2L_2 \eta_1}}{L_2} \le \eta_2 \le \frac{1 + \sqrt{1-2L_2 \eta_1}}{L_2}$. 
Because the learning rate cannot be negative, its value range is truncated by zero. 
And in the interval of $\eta_2$, to avoid the square $\sqrt{1 - 2 L_2 \eta_1}$ from being negative, we let $\eta_1 \le \frac{1}{2L_2}$. All above conditions imply
\begin{equation}\label{eta_choice}
	\begin{aligned}
		\left\{ \begin{array}{cl}
			0 < \eta_1 \le \min\left(\frac{\sqrt{L_1^{2} + 8 L^2} - L_1}{4 L^{2}}, \frac{1}{2L_2}\right) & \\
			\max \left(0, \frac{1 - \sqrt{1 - 2 L_2 \eta_1}}{L_2} \right) < 
			\eta_2
			\le \frac{1 + \sqrt{1 - 2 L_2 \eta_1}}{L_2}
			& \\
			\alpha = \frac{\eta_1}{2} & \\
		\end{array} \right.
	\end{aligned},
\end{equation}
Then, by Assumption \ref{app:assump1} (PL condition) and Eq. (\ref{conditionB}), we have 
\begin{equation*}
	\begin{aligned}
		\mathcal{L}_{2}^{(i, i)} - \mathcal{L}_{2}^{(i+1, i+1)} & \ge \alpha \left(
		\Big\Vert \nabla \mathcal{L}_{2}^{(i, i)} \Big\Vert^{2} - \Vert \bm{\epsilon}^{(i)} \Vert^{2}
		\right)
		& \ge 
		\alpha \mu \left(
		\mathcal{L}_2^{(i, i)} - \mathcal{L}_2^{\star}
		\right) - \alpha \Vert \bm{\epsilon}^{(i)} \Vert^{2}.
	\end{aligned}
\end{equation*}
Re-arranging both sides, we have 
\begin{equation}\label{eq17}
	\begin{aligned}
		\mathcal{L}_{2}^{(i+1, i+1)} - \mathcal{L}_2^{*}
		& \le
		\alpha \Vert \bm{\epsilon}^{(i)} \Vert^{2}
		- \alpha\mu
		\left(
		\mathcal{L}_{2}^{(i, i)} - \mathcal{L}_2^{*}
		\right)
		+ \mathcal{L}_{2}^{(i, i)}
		- 
		\mathcal{L}_2^{*}
		\\ & =
		\alpha \Vert \bm{\epsilon}^{(i)} \Vert^{2} + 
		\left(
		1 - \alpha \mu
		\right)
		\left(
		\mathcal{L}_{2}^{(i, i)} - \mathcal{L}_2^{*}
		\right).
	\end{aligned}
\end{equation}
Recusively performing Eq. (\ref{eq17}), we have
\begin{equation*}
	\begin{aligned}
		\mathcal{L}_{2}^{(i+1, i+1)} - \mathcal{L}_2^{*}
		& \le 
		\left(
		1 - \alpha \mu
		\right)^{i+1}
		\left(
		\mathcal{L}_{2}^{(0, 0)} - \mathcal{L}_2^{*}
		\right)
		+ 
		\alpha \sum_{k=0}^{i} \left(
		1 - \alpha \mu
		\right)^{k} \Vert \bm{\epsilon}^{(i-k)} \Vert^{2}
	\end{aligned}.
\end{equation*}
To ensure the convergence, $1 - \alpha \mu$ should be in $(0, 1)$, which implies
$0 < \alpha < \frac{1}{\mu}$ and accordingly $\eta_1<\frac{2}{\mu}$. Combining the learning rate conditions Eq. (\ref{eta_choice}), we conclude the proof. 
\end{proof}

\section{Proof of Proposition \ref{proposition1}}

\begin{proposition}
	\label{app:proposition1}
	For a model composed of $L$ local layers parameterized by $\vtheta_k$ with their local errors $\gL_{k}$, $k=1,...,L$, when $\gL_k^{SGR}=0$ for all layers $2\le k \le L$ for a batch of data, we have
	\begin{equation}
		\nabla_{\vtheta_k}\gL_L = \nabla_{\vtheta_k}\gL_k,\quad \forall 1\le k \le L-1, \notag
	\end{equation}
	which implies that all local updates are equivalent to learning with the global true gradient back propagated from the last-layer error $\gL_L$. 
\end{proposition}

\begin{proof}[Proof of Proposition \ref{app:proposition1}]
	Since $\gL_L^{SGR}=0$, we have 
	\begin{equation}\label{L-1}
		\frac{\partial \gL_L}{\partial \vx_{L-1}} = \frac{\partial \gL_{L-1}}{\partial \vx_{L-1}}.
	\end{equation}
	Multiplying both sides of Eq. (\ref{L-1}) by $\mJ_f(\vtheta_{L-1})=\frac{\partial \vx_{L-1}}{\partial \vtheta_{L-1}}$, we have
	\begin{equation}
\nabla_{\vtheta_{L-1}}\gL_L = \nabla_{\vtheta_{L-1}}\gL_{L-1}. \notag
	\end{equation}
	Multiplying both sides of Eq. (\ref{L-1}) by $\mJ_f(\vx_{L-2})=\frac{\partial \vx_{L-1}}{\partial \vx_{L-2}}$, and considering $\gL_{L-1}^{SGR}=0$, we have 
	\begin{equation}\label{eq20}
	\frac{\partial \gL_L}{\partial \vx_{L-2}} = \frac{\partial \gL_{L-1}}{\partial \vx_{L-2}}= \frac{\partial \gL_{L-2}}{\partial \vx_{L-2}}.
	\end{equation}
	Similarly multiplying Eq. (\ref{eq20}) the Jacobian matrix of feature $\vx_{L-2}$ \emph{w.r.t.} $\vtheta_{L-2}$, we have $\nabla_{\vtheta_{L-2}}\gL_L = \nabla_{\vtheta_{L-2}}\gL_{L-2}$. Recursively performing this process until the first layer, we have
	\begin{equation}
		\nabla_{\vtheta_{k}}\gL_L = \nabla_{\vtheta_{k}}\gL_{k},\quad 1\le k \le L-1.
	\end{equation}
\end{proof}

\section{Proof of Proposition \ref{proposition2}}

\begin{proposition}
	\label{app:proposition2}
	Consider a two-layer linear model composed of $\vx_1={\vtheta_{1}}\vx_0$ and $\vx_2={\vtheta_{2}}\vx_1$, where ${\vtheta_{1}}$ and ${\vtheta_{2}}$ are the learnable linear matrices, and $\vx_0$ and $\vx_2$ are the input and output of the model, respectively. We use fixed ETF structures $\mM_1$ and $\mM_2$ for the local classifier heads and the cross entropy (CE) loss for local errors $\gL_1(\mM_1\vx_1, y)$ and $\gL_2(\mM_2\vx_2, y)$. Denote ${\gL}'_2$  as the CE loss value in inference after performing one step of gradient descent of $\gL_1$ and $\gL_2$ by Eq. (\ref{update1}) and (\ref{update2}), and denote $\hat{\gL}'_2$ as the one with our $\gL_2^{SGR}$ on the second layer. Assume that the prediction logit for the ground truth label in the second layer is larger than the one  in the first layer, we have $\hat{\gL}'_2\le {\gL}'_2$, when $\gL_2^{SGR}$ is small. 
\end{proposition}

\begin{lemma}
	\label{app:lemma}
	For a fixed ETF classifier $\mM$ defined in Eq. (\ref{etf_classifier}), consider two features $\vx$ and $\hat{\vx}$ belonging to the same class $y$, such that $\hat{\vx}=\vx-\eta\ \delta\vx$, where $\eta>0$, $\delta\vx=-(1-p_y)\vm_y+\sum_{k\ne y}p_k \vm_k$, $\vm_k$ is the classifier vector of $\mM$ for class $k$, and $p_y+\sum_{k\ne y}p_k=1$, $p_y, p_k>0$. When using cross entropy loss $\gL(\mM\vx, y) = -\log\frac{\exp({\vx}^T \vm_y)}{\sum \exp ({\vx}^T \vm_k)}$, we have
	\begin{equation}
		\gL(\mM\hat{\vx}, y) < \gL(\mM\vx, y).
	\end{equation}
\end{lemma}

\begin{proof}[Proof of Lemma \ref{app:lemma}]
	Based on Eq. (\ref{etf_classifier}), we have 
	\begin{align}
		\hat{\vx}^T\vm_y &= {\vx}^T\vm_y + \eta (1-p_y) + \frac{\eta}{K-1}\sum_{k\ne y}p_k \notag\\
		& = {\vx}^T\vm_y  + \frac{K}{K-1}(1-p_y)\eta,  \notag
	\end{align}
	where $K>0$ is the number of total classes. For $k\ne y$, we have
	\begin{align}
		\hat{\vx}^T\vm_k & = {\vx}^T\vm_k - \eta(1-p_y) \frac{1}{K-1}- \eta p_k+\eta \frac{1}{K-1}\sum_{k'\ne y, k'\ne k}p_k\\
		&={\vx}^T\vm_k - \frac{K}{K-1}p_k\eta.
	\end{align}
	Since $0<p_i<1, \forall 1\le i \le K$, we have
	\begin{align}
		\gL(\mM\hat{\vx}, y) &= -\log \frac{\exp(\hat{\vx}^T \vm_y)}{\sum \exp (\hat{\vx}^T \vm_k)} \notag\\
		& = -\log \frac{\exp\left({\vx}^T\vm_y  + \frac{K}{K-1}(1-p_y)\eta\right)}{\exp \left({\vx}^T\vm_y  + \frac{K}{K-1}(1-p_y)\eta\right) + \sum_{k\ne y} \exp ({\vx}^T\vm_k - \frac{K}{K-1}p_k\eta)} \notag \\
		& < -\log \frac{\exp\left({\vx}^T\vm_y \right)}{\exp \left({\vx}^T\vm_y \right) + \sum_{k\ne y} \exp ({\vx}^T\vm_k)} = \gL(\mM{\vx}, y), \notag
	\end{align}
	which concludes the proof of Lemma \ref{app:lemma}.
\end{proof}

\begin{proof}[Proof of Proposition \ref{app:proposition2}]
	In the first linear layer $\vx_1=\vtheta_1\vx_0$, we have the gradient of $\gL_1(\mM_1\vx_1, y)$ \emph{w.r.t.} $\vtheta_1$ as
	\begin{equation}
		\nabla_{\vtheta_{1}} \gL_1 = 
		\begin{bmatrix}
			\vx_0^T \\
			\vdots \\
			\vx_0^T
		\end{bmatrix} \otimes \delta \vx_1, 
	\end{equation}
	where $\otimes$ performs multiplication for each row vector $\vx_0^T$ and each element of $\delta \vx_1$  and $\delta \vx_1$ is the gradient from $\gL_1$ \emph{w.r.t.} $\vx_1$:
	\begin{equation}
		\delta \vx_1 = -(1-p^{(1)}_y) \vm^{(1)}_y + \sum_{k\ne y} p^{(1)}_k \vm^{(1)}_k,
	\end{equation}
	where $\vm^{(1)}_k$ is the classifier vector of $\mM_1$ for class $k$, and $p_k^{(1)}$ is the probability predicted by the first layer for class $k$. Then we have $\vx'_{1}$ after the update of $\vtheta_1$ as:
	\begin{equation}
		\vx'_{1} = \vtheta'_1 \vx_0=(\vtheta_1 - \eta \nabla_{\vtheta_{1}} \gL_1)\vx_0=\vx_1 - \eta \left\| \vx_0 \right\|^2 \delta\vx_1,
	\end{equation}
	where $\eta$ is the learning rate.
	
	Similarly, in the second layer $\vx_2=\vtheta_{2} \vx_1$, when our method is not adopted, we have
	\begin{equation}
		\vtheta'_2 = \vtheta_2 - \eta
			\begin{bmatrix}
			\vx_1^T \\
			\vdots \\
			\vx_1^T
		\end{bmatrix} \otimes \delta \vx_2, 
	\end{equation}
	where  $\delta \vx_2$ is the gradient from $\gL_2$ \emph{w.r.t.} $\vx_2$. In inference, we have the new output as
	\begin{equation}
		\vx'_2 = \vtheta'_2 \vx'_1 = \vx_2 - \eta\left\|\vx_0\right\|^2\vtheta_2\delta\vx_1-\eta\left\| \vx_1\right\|^2\delta\vx_2+\eta^2\left\|\vx_0\right\|^2
		\begin{bmatrix}
			\vx_1^T \\
			\vdots \\
			\vx_1^T
		\end{bmatrix} \otimes \delta \vx_2 \cdot \delta \vx_1.
	\end{equation}
	When our SGR is adopted, we have
	$$
	\gL_2^{(SGR)}=\frac{1}{2}\left\| \frac{\partial \gL_1}{\partial \vx_1} - \frac{\partial \gL_2}{\partial \vx_1}\right\|^2=\frac{1}{2}\left\| \delta \vx_1 - \vtheta_2^T \delta \vx_2\right\|^2.
	$$
	In our case, we denote the updated second layer parameter as $\hat{\vtheta}'_2$, which can be formulated as
	\begin{equation}
		\hat{\vtheta}'_2 = \vtheta_2 - \eta
		\begin{bmatrix}
			\vx_1^T \\
			\vdots \\
			\vx_1^T
		\end{bmatrix} \otimes \delta \vx_2 - \eta \delta\vx_2\left(  \delta \vx_1 - \vtheta_2^T \delta \vx_2 \right)^T.
	\end{equation}
	Accordingly, we denote the new output in inference in our case as $\hat{\vx}'_2$, which can be formulated as
	\begin{align}
		\hat{\vx}'_2 = \hat{\vtheta}'_2 \vx'_{1}  & = \vx'_2 -  \eta \delta\vx_2\left(  \delta \vx_1 - \vtheta_2^T \delta \vx_2 \right)^T \left( \vx_1 - \eta \left\| \vx_0 \right\|^2 \delta\vx_1 \right) \\
		& = \vx'_2 -  \eta \delta\vx_2 \left(  \delta \vx_1^T \vx_1-\eta \left\| \vx_0 \right\|^2 \left\| \delta\vx_1\right\|^2  -(\vtheta_2^T \delta \vx_2 )^T\vx_1 +\eta \left\| \vx_0 \right\|^2 (\vtheta_2^T\delta \vx_2 )^T\delta\vx_1    \right).
	\end{align}
	We assume that $\gL_2^{(SGR)}$ is small, which indicates that $ \left\| \delta\vx_1\right\|^2 \approx (\vtheta_2^T\delta \vx_2 )^T\delta\vx_1$, and thus we have
	\begin{equation}
			\hat{\vx}'_2  \approx \vx'_2 -  \eta \delta\vx_2 \left(  \delta \vx_1^T \vx_1  -(\vtheta_2^T \delta \vx_2 )^T\vx_1 \right)=\vx'_2 -  \eta \delta\vx_2 \left(  \delta \vx_1^T \vx_1  -\delta \vx_2 ^T\vx_2 \right).
	\end{equation}
	Because we assume that the prediction logit for the label class in the second layer is larger than the one in the first layer, based on the definition of ETF classifier, we have $\delta \vx_2 ^T\vx_2 \le \delta \vx_1^T \vx_1$. Consequently, the output feature in inference using our method is in a form of $\hat{\vx}'_2 =  {\vx}'_2 -\beta \delta \vx_2$ where $\beta \ge 0$. Considering the conclusion of Lemma \ref{app:lemma}, we have 
	\begin{equation}
		\hat{\gL}'_2(\hat{\vx}'_2) \le {\gL}'_2({\vx}'_2), 
	\end{equation}
	which concludes the proof of Proposition \ref{app:proposition2}.
\end{proof}

\section{Analysis of Computation Cost of SGR}

Eq. (\ref{sgrloss}) introduced by our method requires calculating the gradient of $\delta x$, which is the gradient of the local error \emph{w.r.t.} the input $x$ of this block. 
However, compared with InfoPro, our method is not obviously slower, and in some cases even faster (ResNet 50 and 101 when $K$=3) while being more efficient in memory consumption. That is because InfoPro extra introduces heavy reconstruction heads composed of multiple layers (e.g. 4 convolution layers on ImageNet) other than the local classification heads, while our method only uses light local classifiers.

Although our method concerns the gradient calculation of Eq. (\ref{sgrloss}), which includes a second-order derivative, here we provide its analytical computation cost in a block to show that the introduced cost is bearable.

Consider a local block $y=\sigma(Wx)$, where $x\in\mathcal{R}^n$ is the input of this block, $y\in\mathcal{R}^m$ is the output of this block, $W\in\mathcal{R}^{m\times n}$ is the parameter, and $\sigma$ is the ReLU nonlinear activation. Its local error is given by $\mathcal{L}(y, \mathcal{Y})$, where $\mathcal{L}$ is the loss function and $\mathcal{Y}$ is the ground truth. Our SGR loss term is in the form of $\mathcal{L}^{SGR} = \frac{1}{2}\| \frac{\partial \mathcal{L}}{\partial x} - g\|^2_2$, where $g$ is the gradient from the previous local error towards $x$ and thus is a constant vector here. We have $\frac{\partial \mathcal{L}}{\partial x} = W^T\left[ \frac{\partial \mathcal{L}}{\partial y}\otimes \sigma'(Wx) \right]$, where $\otimes$ is the element-wise multiplication between two vectors, $\sigma'$ is the function of the first-order derivative of $\sigma$. Because the second-order derivative of ReLU, i.e. $\sigma''$ would be zero, we have $$ \frac{\partial}{\partial W}\left(\frac{\partial \mathcal{L}}{\partial x}\right)_i=\left[\mathbf{0}_m,\dots, \frac{\partial \mathcal{L}}{\partial y}\otimes \sigma'(Wx), \dots, \mathbf{0}_m\right]\in \mathcal{R}^{m\times n}, $$ where $i$ denotes the $i$-th element of $\frac{\partial \mathcal{L}}{\partial x}$, and $\frac{\partial \mathcal{L}}{\partial y}\otimes \sigma'(Wx)$ lies in the $i$-th column of the gradient matrix above with all the other columns as $\mathbf{0}_m$.

Then we have the gradient of our SGR loss \emph{w.r.t.} $W$ as: 
\begin{equation}
\frac{\partial \mathcal{L}^{SGR}}{\partial W} = \left[ \frac{\partial \mathcal{L}}{\partial y}\otimes \sigma'(Wx) \right]\left(\frac{\partial \mathcal{L}}{\partial x} - g\right)^T = \left[ \frac{\partial \mathcal{L}}{\partial y}\otimes \sigma'(Wx) \right]\left(W^T\left[ \frac{\partial \mathcal{L}}{\partial y}\otimes \sigma'(Wx) \right] - g\right)^T. 
\label{computation}
\end{equation}
From the equation above we can see that although our method concerns second-order derivatives, the gradient calculation of our $\mathcal{L}^{SGR}$ is simply matrix multiplications with $\frac{\partial \mathcal{L}}{\partial y}\in\mathcal{R}^m$, $\sigma'(Wx)\in\mathcal{R}^m$, and $W\in\mathcal{R}^{m \times n}$. The FLOPS of the equation above is listed in Table \ref{flops}.

\begin{table}[!h]  
	\vspace{-2mm}
	\caption{Computation analysis in Eq. (\ref{computation})}	
	\begin{center}
			\begin{tabular}{lc}
				\toprule
				term & FLOPs \\
				\midrule
				$A=\left[ \frac{\partial \mathcal{L}}{\partial y}\otimes \sigma'(Wx) \right]$ &	$O(m)$ \\
				$B=\left(W^TA - g\right)^T$	& $O((2m-1)n+n)$ \\
				$A\cdot B$ & $O(mn)$ \\
				total & $O(3mn + m)$\\
				\bottomrule
			\end{tabular}
		\label{flops}
	\end{center}
	\vspace{-3mm}
\end{table}

Therefore, the theoretical computational cost is $O(3mn+m)$, which is only in a similar scale of linear layer computation. That is why training with our method will not be severely slowed down compared with global BP training. Additionally, we provide the following two strategies that can further accelerate training by leveraging the advantages brought by our method.

(1) Due to the high memory efficiency of our method, we can use a larger batchsize to speedup training from improved GPU parallelism utilization.
(2) Because our method detaches all the other blocks when training each block locally, we support asynchronous training for all the blocks. Suppose we have 3 local blocks, $f_1, f_2, f_3$, for a neural network, the forward propagation and local training of $f_1(x^{(t+2)})$, $f_2(x^{(t+1)})$, $f_3(x^{(t)})$ can be performed asynchronously, where $x^{(t)}$ denotes the batch of train data in iteration $t$. This strategy can also speedup training.

\section{Implementation Details}
\label{training_details}

\textbf{Training details}

We train all models following the common practices. 
To train ResNet models, we use the SGD optimizer with a learning rate of 0.1, a momentum of 0.9, and weight decay of 0.0001. We train these networks for 100 epochs with a batch size of 1024. The initial learning rate is set to 0.1 and decreases by a factor of 0.1 at epochs 30, 60, and 90. Data preprocessing includes random resizing, flipping, and cropping.
For ViT-S/16 training on ImageNet, we use a batch size of 4096. We use the AdamW optimizer with a learning rate of 0.0016, and we apply a cosine learning rate annealing schedule after a linear warm-up for the first 20 epochs. The training process lasts for 300 epochs, and we apply data augmentations like random resized cropping, horizontal flipping, RandAugment, and Random Erasing.
Training Swin Transformers on ImageNet uses a batch size of 1024. We use the AdamW optimizer with a learning rate of 0.001, along with betas (0.9, 0.999), epsilon 1e-08, and a weight decay of 0.05. Learning rate scheduling includes a linear warm-up for the first 20 epochs, followed by a cosine annealing schedule with a minimum learning rate of 1e-05. For Swin Transformer models, including ``tiny'' and ``small'' versions, we have a dropout rate of 0.2, an input image size of $224\times 224$, and a patch size of $16\times 16$. Training runs for 300 epochs, and we incorporate data augmentations like random resized cropping, horizontal flipping, RandAugment, and Random Erasing. The coefficient of our SGR loss $\lambda$ is set as 10k in these supervised learning experiments on ImageNet.

In our self-supervised experiments, we follow the Barlow Twins training procedure~\cite{zbontar2021barlow}. We use a ResNet-50 model and train it for 300 epochs. We optimize it using LARS with a learning rate of 1.6, a momentum of 0.9, and a weight decay of 1e-06. The batch size is 2048, and the learning rate schedule includes a linear warm-up for the first 10 epochs, followed by cosine annealing until the 300th epoch, maintaining a minimum learning rate of 0.0016. We use a three-layer MLP projector with 8192 hidden units and 8192 output units. To evaluate the transferability of feature representations on the ImageNet dataset through linear classification, we employ the SGD optimizer with a learning rate of 0.3, a momentum of 0.9, and a weight decay of 1e-06. We train the model using a cosine annealing learning rate schedule over 100 epochs, with a batch size of 256. Data augmentation techniques include random resizing and flipping. The model architecture is based on a ResNet-50 backbone with specific stages frozen, along with a linear classification head. 

On CIFAR, we train all models for 200 epochs with an initial learning of 0.1 and a cosine annealing learning rate scheduler. We use a batchsize of 128 and adopt the SGD optimizer with a momentum of 0.9 and a weight decay of 5e-4. Standard data pre-processing and augmentations are adopted. The coefficient of our SGR loss $\lambda$ is set as 1 as default.

\textbf{Local classifier setups}

For experiments on ImageNet, we divide a model into 2 or 3 local modules such that the training memory consumption reaches its minimum. We adopt the same local classifier following \cite{wang2021revisiting} for fair comparison. For experiments on CIFAR, when $K=16$, each residual block and the initial stem layer is a local module. When $K=3$, the model is divided according to the feature spatial resolution. When $K=2$, we keep the last stage (feature size of $8\times8$) as one local module, and all the other blocks as another local layer. We adopt the same local classifier following \cite{belilovsky2019greedy} ($k=2$ in their paper) for fair comparison. There are two layers in the classifier including one convolution layer and one linear layer for classification.

\begin{figure*}[t]
	\begin{subfigure}{0.33\textwidth}
		\centering
		\includegraphics[width=1\linewidth]{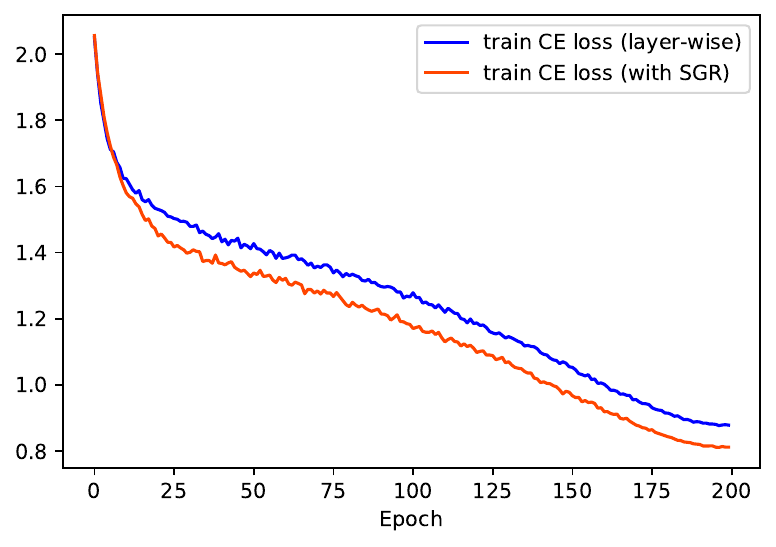}  
		\label{train_loss}
	\end{subfigure}
	\begin{subfigure}{0.33\textwidth}
		\centering
		\includegraphics[width=1\linewidth]{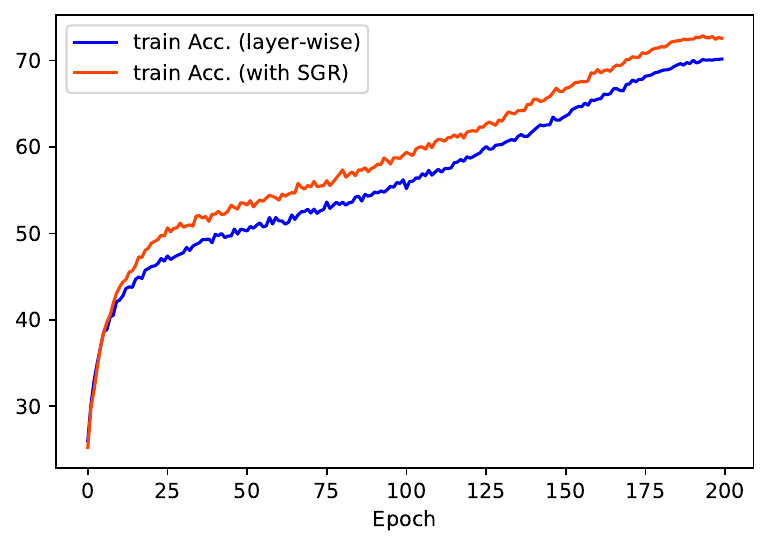}  
		\label{fig:train_acc}
	\end{subfigure}
	\begin{subfigure}{0.33\textwidth}
		\centering
		\includegraphics[width=1\linewidth]{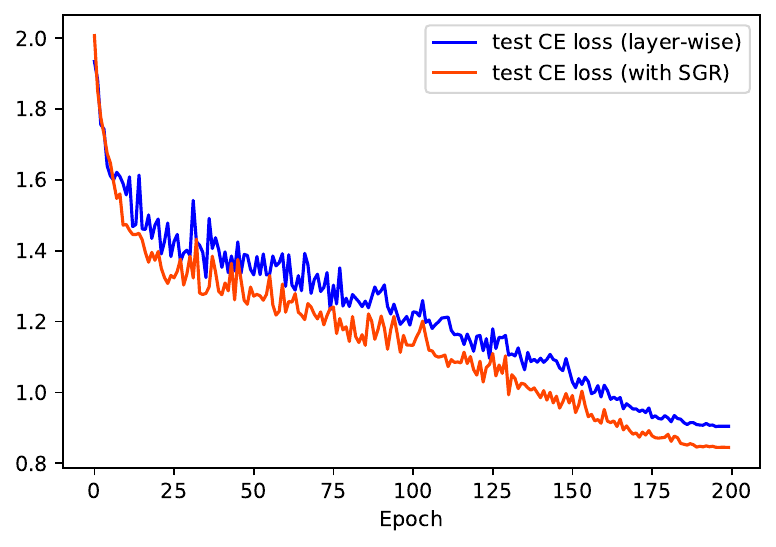}  
		\label{testloss}
	\end{subfigure}
	\vspace{-10mm}
	\caption{Train loss (left), train accuracy (middle), and test loss (right) curves with and without our method. The corresponding SGR loss value and test accuracy curves are shown in Figure \ref{curve}.}
	\label{traintestcurve}
\end{figure*}


\section{Others}

We provide more discussions and results, some of which are suggested by reviewers in the review process. 

\textbf{The limitations of our work include:} 1) The method helps local training to better approach to global BP training by mitigating the discordant local updates successively. But there is still a gap with global BP training because the regularization term could not be optimized to 0 in real implementation. 2) The analysis and method in this study mainly deal with local training from scratch. How to apply our method into finetuning a pretrained model locally, especially for large models, deserves future exploration. 3) The train time of our method is on par with InfoPro, but is still longer than global BP training due to the introduced regularization term.

\textbf{Further insights into how this regularization affects the convergence behavior and final performance.} Our method adds a regularization onto the classification loss. It seems that the regularization will impede the original optimality of the classification loss and deteriorate the final performance. This is true for one-layer network. After all the local training of one layer is just the global BP training of this network. However, our method is applied in the local training of multiple-layer networks. For the last layer that outputs the final representation, its classification loss is dependent on two variables, the parameters of the last layer, and the input of the last layer. In global BP training, the chain rule can pass the gradient of the last-layer loss \emph{w.r.t.} the input of the last layer into prior layers to change their parameters and produce a better input feature that decreases the last-layer loss. But in local training, gradients of all layers are isolated. We have no way to ensure that the input of the last layer, \emph{i.e.} the output of its previous layer optimized by the previous local error, most satisfies the demand of minimizing the last-layer loss. In this case, from a perspective of optimization, our method does not solely optimize the last-layer parameters with its classification loss, but moves the parameters to a direction, such that the change of its input feature caused by the previous local error also enables to minimize the last-layer classification loss. That is to say training with our regularization $\mathcal{L}^{SGR}$ simultaneously optimizes the parameter variable and the input variable, to induce a larger reduction of the classification loss, and a better model performance finally. The input variable is optimized in an implicit manner, by mitigating the discordant local updates from the top to the end of a deep network successively. As shown in Proposition \ref{proposition1}, when $\mathcal{L}^{SGR}=0$ for all layers, each local update will be equivalent to the true gradient directly from the last-layer classification loss, in which situation, the input change of each local layer most satisfies the demand to minimize the last-layer classification loss. 

\textbf{Discussion of the comparison to other memory-saving techniques.} Other memory-saving techniques, including gradient checkpointing and reversible architectures, are based on the trade-off between memory cost and computation cost. These methods still rely on global BP training, which means the activations of all layers are required when performing backpropagation of the last-layer error. In gradient checkpointing, during the forward propagation, the prior activations are removed once it is propagated into the next layer to save memory cost. In the backward propagation, however, the activations are recovered by performing forward propagation once again for each layer. Similarly, reversible architectures recover input activations of each layer in the backward propagation by re-performing these layers with their output activations using a dual path. Therefore, both gradient checkpointing and reversible architecture require to re-calculate each layer, which heavily increases their computational burden.
In contrast, our method belongs to local training. It naturally enjoys memory efficiency because the forward and backward propagations of each layer are performed locally. After the local update of one layer, the activations of this layer can be detached in this iteration without performing the operations inside this layer again. Therefore, we only need to consume the memory for one layer's update while performing the forward and backward propagations of each layer only once.

\textbf{More results.} A PyTorch-like pseudocode for the implementation of our SGR is shown in Algorithm \ref{alg:example}. The training curves, including train loss, train accuracy, and test loss, are shown in Figure \ref{traintestcurve}. The corresponding SGR loss value and test accuracy curves are shown in Figure \ref{curve}.

\begin{algorithm}[tb]
	\caption{A PyTorch-like pseudocode for local training with SGR}
	\label{alg:example}
	\begin{algorithmic}
		\STATE {\bfseries Given:} a model divided into a module list [$\gM_k$] parameterized by $\vtheta_{k}$ with local heads $\gL_k(\cdot, Y)$, $1\le k \le L$; train data $\gS \in \left\{ (\vx_0^t, Y^t) \right\}_{t\le T}$ of samples or mini-batches; hyper-parameter \verb|lambda|;
		\STATE \verb|optimizer_list = [torch.optim.SGD(|$\vtheta_k$\verb|) for k in range(1,L+1)]|
		\STATE \verb|SGR_criterion = nn.MSELoss()|
		\FOR{$(\vx_0^t, Y^t)\in\gS$}
		\STATE \verb|input| = $\vx_0^t$
		\FOR{$k=1$ {\bfseries to} $L$}
		\STATE \verb|optimizer = optimizer_list[k]|
		\STATE \verb|input.requires_grad_()| 
		\STATE \verb|output = |$\gM_k$\verb|(input)|
		\STATE \verb|loss_cls = |$\gL_{k}(\verb|output|, Y)$
		\IF{$k >= 2$}
		\STATE \verb|delta_now = torch.autograd.grad(|
		\STATE \quad \quad \quad \quad \quad \quad \verb|loss_cls, input, retain_graph=True, create_graph=True)[0]|
		\STATE \verb|delta_pre_norm = torch.flatten(delta_pre, 1)|
		\STATE \verb|delta_pre_norm = delta_pre_norm / torch.sqrt(|
		\STATE \quad \quad \quad \quad \quad \quad \verb|torch.sum(delta_pre_norm ** 2, dim=1, keepdims=True))|
		\STATE \verb|delta_now_norm = torch.flatten(delta_now, 1)|
		\STATE \verb|delta_now_norm = delta_now_norm / torch.sqrt(|
		\STATE \quad \quad \quad \quad \quad \quad \verb|torch.sum(delta_now_norm ** 2, dim=1, keepdims=True))|
		\STATE \verb|loss_SGR = SGR_criterion(delta_now_norm, delta_pre_norm)|
		\STATE \verb|loss = loss_cls + lambda * loss_SGR|
		\STATE \verb|optimizer.zero_grad()|
		\STATE \verb|loss.backward()|
		\STATE \verb|optimizer.step()|
		\ELSE
		\STATE \verb|delta_pre = torch.autograd.grad(|
		\STATE \quad \quad \quad \quad \quad \quad \verb|loss_cls, output, retrain_graph=True)[0].detach()| 
		\STATE \verb|optimizer.zero_grad()|
		\STATE \verb|loss_cls.backward()|
		\STATE \verb|optimizer.step()|
		\ENDIF
		\STATE \verb|input = output.detach()|
		\ENDFOR
		\ENDFOR
	\end{algorithmic}
	\label{algorithm}
\end{algorithm}

\end{document}